\documentclass{article} 

\usepackage{arxiv}
\usepackage{amssymb, amsmath}
\usepackage{amsthm}
\usepackage{enumitem}
\usepackage[utf8]{inputenc} 
\usepackage[T1]{fontenc}    
\usepackage{hyperref}       
\usepackage{url}            
\usepackage{booktabs}       
\usepackage{amsfonts}       
\usepackage{graphicx}       
\usepackage{nicefrac}       
\usepackage{microtype}      
\usepackage{xcolor}         
\usepackage{cancel}
\usepackage{subcaption}
\usepackage{wrapfig}
\usepackage{multirow}
\usepackage{tikz}
\usetikzlibrary{positioning}

\newcommand{\taco}[1]{\texttt{TACO-Net}}
\newtheorem{theorem}{Theorem}

\title{TACO-Net: Topological Signatures Triumph in \\3D Object Classification}

\date{}

\author{%
  Anirban Ghosh, Ayan Dutta \\
  School of Computing\\
  University of North Florida\\
  Jacksonville, FL \\
  \texttt{anirban.ghosh,a.dutta@unf.edu} \\
}

\begin{document}

\maketitle
\begin{abstract}
  3D object classification is a crucial problem due to its significant practical relevance in many fields, including computer vision, robotics, and autonomous driving. Although deep learning methods applied to point clouds sampled on CAD models of the objects and/or captured by LiDAR or RGBD cameras have achieved remarkable success in recent years, achieving high classification accuracy remains a challenging problem due to the unordered point clouds and their irregularity and noise. To this end, we propose a novel state-of-the-art (SOTA) 3D object classification technique that combines topological data analysis with various image filtration techniques to classify objects when they are represented using point clouds. We transform every point cloud into a voxelized binary 3D image to extract distinguishing topological features. Next, we train a lightweight one-dimensional Convolutional Neural Network (1D CNN) using the extracted feature set from the training dataset. Our framework, \taco~, sets a new state-of-the-art by achieving $99.05\%$ and $99.52\%$ accuracy on the widely used synthetic benchmarks ModelNet40 and ModelNet10, and further demonstrates its robustness on the large-scale real-world OmniObject3D dataset. When tested with ten different kinds of corrupted ModelNet40 inputs, the proposed \taco~ demonstrates strong resiliency overall. 
\end{abstract}

\section{Introduction}

\begin{figure}[h]

    \centering
    \includegraphics[width=1\linewidth]{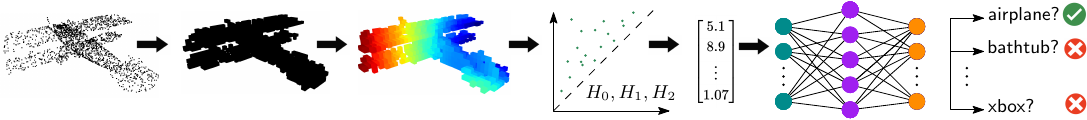}
    \caption{An airplane point cloud (from ModelNet40) is converted into a 3D binary image, which is then transformed into a set of 3D grayscale images (just one is shown) using different filtration techniques. Every grayscale image admits a separate cubical persistence. A feature vector of length $36$ is obtained for every persistence. The vectors are then concatenated to form the final feature vector for the plane. The feature vectors are finally trained using a 1D CNN for object classification. }
    \label{fig:intro}

\end{figure}

Semantic object recognition is one of the most fundamental capabilities that modern-day autonomous systems, such as robots and cars, demand to operate in dynamic real-world environments. Given an input point cloud, the objective is to classify the input into one of the known categories~\cite{maturana2015voxnet,qi2017pointnet,wu20153d}. Such object classification is critical in numerous real-world applications, including autonomous driving, robotics, augmented reality, and 3D scene understanding. In recent years, deep learning techniques have achieved remarkable success in the 3D object classification task by using various modes of inputs such as voxels~\cite{maturana2015voxnet}, multi-views~\cite{kanezaki2019rotationnet}, and raw point clouds~\cite{qi2017pointnet}, to name a few. 

Unlike 2D images, point clouds provide rich geometric and spatial information in three dimensions, enabling more precise object recognition in complex environments~\cite{sarker2024comprehensive}. Further, unlike 2D images, point clouds enable a mobile robot, for example, to recognize objects in various ambient light and weather conditions efficiently. However, point cloud-based object classification remains challenging due to the unstructured and sparse nature of point cloud data, which lacks a regular grid structure and often suffers from noise, occlusion, and varying point densities~\cite{ben20183dmfv,sun2022benchmarking,uy2019revisiting}. Furthermore, the permutation invariance of points and the need to capture both local and global geometric features add layers of complexity to model design~\cite{zhao2022rotation}. These factors demand innovative methodologies that effectively learn from unordered and irregular 3D geometric data, driving continued research and development in this field.

Most of the existing approaches use a deep machine learning framework where the input is either a set of pictures of the object, a raw set of points, or volumetric shape of the object~\cite{su2015multi,feng2018gvcnn,kanezaki2019rotationnet,qi2017pointnet,qi2017pointnet++,qian2022pointnext,maturana2015voxnet}. Unlike these, we take a novel approach where the $n$-element point clouds sampled from the 3D objects (both from training and test sets) are transformed into voxelized 3D binary images to extract features from it using topological data analysis (TDA) via cubical persistence (defined in Section~\ref{sec:tda}). 

We deploy a 1D CNN, trained using the topological feature vectors obtained for the 3D objects from the training set for the class prediction task. See Fig.~\ref{fig:intro} for an overview of our approach.
Although TDA has a direct connection with shapes, surprisingly, TDA has not been successfully used for large-scale 3D object classification. Similar to the challenges associated with designing deep learning models using existing network layers, finding an effective TDA pipeline presents a challenge. First, we test the proposed \textbf{\underline{t}}opological data \textbf{\underline{a}}nalysis-based obje\textbf{\underline{c}}t classificati\textbf{\underline{o}}n framework, named \taco~, on {ModelNet40} and {ModelNet10} datasets. Our experiments show that we achieve SOTA accuracy for both these datasets. Furthermore, we have chosen ten common corruptions to test the robustness of \taco~. The corrupted dataset is tested on the model trained with the uncorrupted {ModelNet40} dataset. Two different levels of corruption have been used in our experiments. Results show that the proposed \taco~ yields high accuracy for all but one corruption type at a low level while achieving moderate accuracy with highly corrupted test data. Further, when tested on a real-world dataset, namely OmniObject3D~\cite{wu2023omniobject3d}, \taco~ again achieved the highest accuracy. To show the generalizability of \taco~, we tested it on two 3D medical object datasets, namely VesselMNIST3D~\cite{yang2020intra} and AdrenalMNIST3D~\cite{yang2023medmnist}, where \taco~ surpassed the highest accuracies and F1-scores of numerous existing techniques such as PointNet~\cite{qi2017pointnet}, PointNet++~\cite{qi2017pointnet++}, and DGCNN~\cite{wang2019dynamic}, among others. The main contributions of our paper are as follows.

\begin{itemize}[leftmargin=*]\itemsep-2pt
\item To the best of our knowledge, this is the first work that uses TDA through cubical persistence to extract features from input point clouds before learning those features using a 1D CNN.
    \item Our proposed novel \taco~ framework achieves higher overall accuracies in both 10 and 40-class variations of the ModelNet dataset - thereby providing a new SOTA performance.
    \item Results show that \taco~ is robust against common types of point cloud corruptions while being easily generalizable to various real-world 3D object datasets.
\end{itemize}

\textbf{Related Work. }Three main types of approaches are prevalent in the 3D object classification literature: voxel-based, multi-view imaging-based, and raw point-based~\cite{sarker2024comprehensive}. Many hybrid methods combine one or more of the above-mentioned techniques. In voxel-based methods, features of the volumetric representation of input point clouds are learned and classified~\cite{maturana2015voxnet}. One of the earliest approaches in this direction is 3D Shapenets~\cite{wu20153d}. Although the objects to be classified are in 3D, taking 2D pictures of them from various angles and classifying those 2D pictures instead has gained attention through MVCNN by Su et al.~\cite{su2015multi}, where they used $80$ pictures of each 3D object. GVCNN~\cite{feng2018gvcnn} improved upon MVCNN by using only $8$ images. MHBN~\cite{yu2018multi}, on the other hand, used only $6$ views of the object, but managed to achieve high mean class accuracy. Usually, convolutional neural networks are used for these 2D image classification techniques. In \cite{ma2018learning}, the authors have used a recurrent neural module along with CNNs. Hypergraph learning has been highly effective for object classification, as shown in \cite{feng2019hypergraph}. One of the highest accuracy yielding approaches, RotationNet~\cite{kanezaki2019rotationnet}, also uses multiple views of the objects, albeit these are unsupervised viewpoints. Point-based approaches are most popular - they take raw point clouds as inputs and learns from their unstructured format, which makes them robust against corruption~\cite{qi2017pointnet,qi2017pointnet++,li2018pointcnn}. One of the pioneering works in this direction is PointNet~\cite{qi2017pointnet}. PointNet++~\cite{qi2017pointnet++} improved upon PointNet by capturing local geometric structures. Transformer-based learning strategies have received attention as well~\cite{zhao2021point}. Graph neural networks have been successful in classifying point cloud objects~\cite{wang2019dynamic,mohammadi2021pointview}. Unlike these, our novel methodology extracts topological features, which are learned by a 1D CNN for object classification and achieves SOTA accuracy.

This paper uses TDA features extracted from the point clouds for object classification. Such a TDA-based approach has been previously used for MNIST data classification~\cite{garin2019topological}. TDA has recently been used to solve a diverse range of problems, such as in medical imaging~\cite{singh2023topological}, biomedicine~\cite{skaf2022topological}, oncology~\cite{bukkuri2021applications}, and cybersecurity~\cite{akcora2020bitcoinheist}.

\section{Description of \texttt{TACO-Net}}\label{sec:tda}

We transform every train and test point cloud $P$ into a 3D binary image, where every active voxel (represented using a $1$) contains at least one point from $P$. Our experiments determine a suitable value for the voxel size of the 3D images. The primary purpose of converting point clouds to 3D binary images is to use different 3D grayscale image filtration techniques to extract distinguishing topological information about the point clouds through their corresponding 3D grayscale images, using TDA. In what follows, we present an overview of the theoretical underpinnings of TDA, leveraged to develop~\taco~.

\textbf{Filtration types}~\cite{garin2019topological}. Let $\mathcal{B}:I \subseteq \mathbf{Z}^3 \rightarrow \{0,1\}$ be a 3D binary image, where every $p\in I$ is a voxel. A voxel is \textit{activated} if its value is $1$; otherwise, it is \textit{deactivated}. A grayscale filtration converts $\mathcal{B}$ into a grayscale 3D image $\mathcal{G}: I \subseteq \mathbf{Z}^3 \rightarrow \mathbf{R}$. Such filtrations can highlight different topological features in the binary image, even visually. We briefly describe the six kinds of filtrations used in~\taco~ to obtain a set of grayscale 3D images for every 3D binary image (constructed for every point cloud) for extracting topological feature vectors. Owing to the difference in the filtration functions, every filtration tends to highlight different features of $\mathcal{B}$. Refer to Fig.~\ref{fig:filters} for an illustration. 

\begin{figure}[h]
  \centering
  \begin{tabular}{cccccccc}
  \includegraphics[scale=0.05]{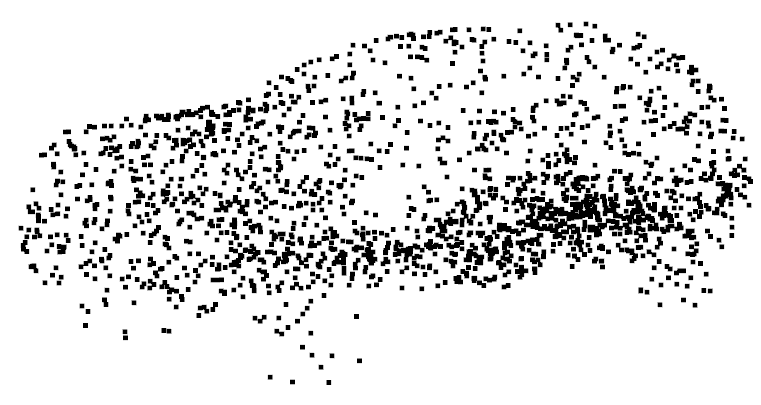} & 
      \includegraphics[scale=0.05]{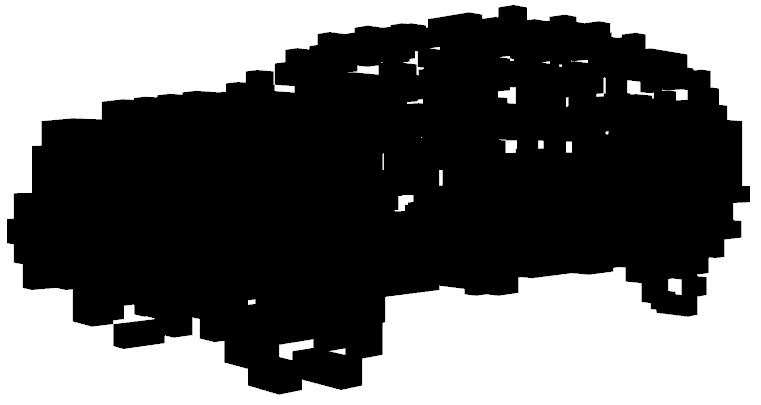} & \includegraphics[scale=0.05]{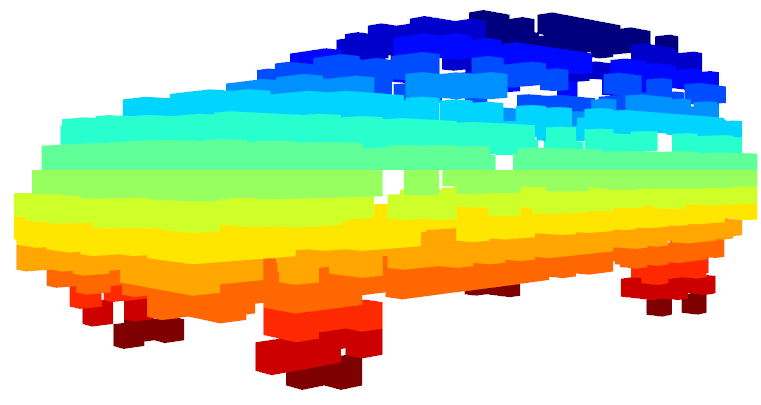} & \includegraphics[scale=0.05] {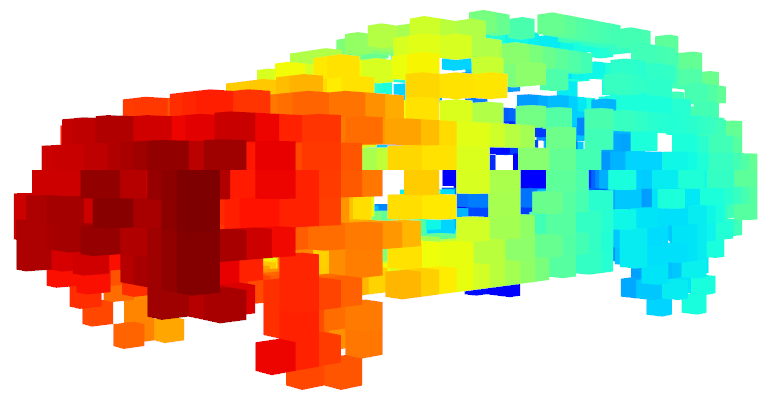} &
      \includegraphics[scale=0.05]{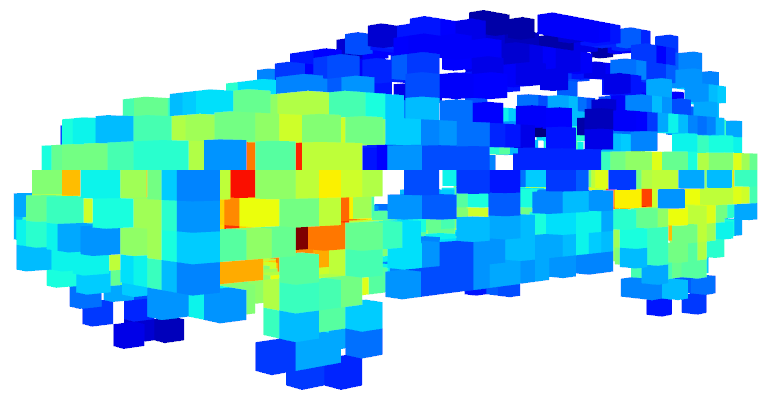} &
      \includegraphics[scale=0.05]{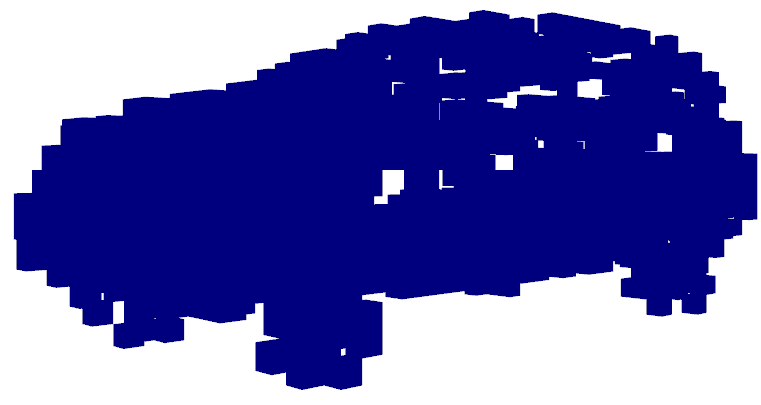} &
      \includegraphics[scale=0.05]{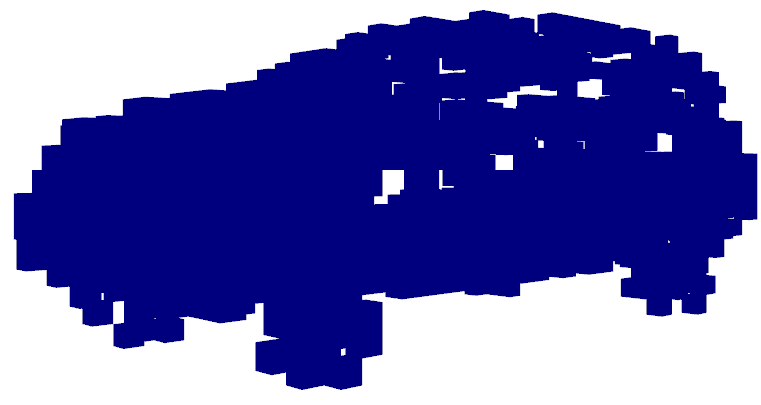} &
      \includegraphics[scale=0.05]{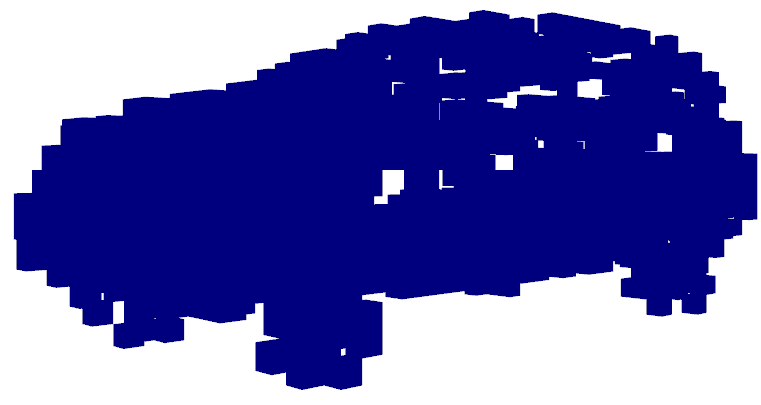}\\
      (a) & (b) & (c) & (d) & (e) & (f) & (g)  & (h)
  \end{tabular}
   
  \caption{A $2048$-element point cloud of a car from ModelNet40 is shown in (a) and its 3D binary image with voxel size $0.05$ in (b). The grayscale images obtained after height using $v:(-1,0,0)$, radial using $c:(4,4,10)$, density, dilation, erosion, and signed distance filtrations, are shown in (c), (d), (e), (f), (g), (h), respectively. Hotter voxels have higher grayscale values. For brevity, voxels outside the shape are not shown; consequently, (f), (g), and (h) appear almost the same.}
  \label{fig:filters}
\end{figure}

\textit{Height filtration.} It needs a direction vector $v$ in $3$-space. Every activated voxel $p \in \mathcal{B}$ is assigned a grayscale value that equals the distance between $p$ and the hyperplane defined by $v$. Every deactivated voxel is assigned the maximum distance between any voxel of $\mathcal{B}$ and the hyperplane defined by $v$, plus one. For \taco~, we have considered the $26$ direction vectors in $\{\{0,1,-1\}^3 \}\setminus \{ (0,0,0) \}$.

\textit{Radial filtration.} A reference voxel $c$, called \textit{center}, is supplied. Every activated voxel $p \in \mathcal{B}$ is assigned the distance between $p$ and $c$. The deactivated voxels are assigned the maximum distance between $c$ and any voxel, plus one. For \taco~, $27$ centers $c_1,c_2,\ldots,c_{27}$ have been considered, chosen as the $27$  vertices of a $3\times 3 \times 3$ grid $\Xi_\mathcal{B}$ inside $\mathcal{B}$. We note that $\mathcal{C}_1 :=[c_1,\ldots,c_9]$ belong to the first vertical slice of $\Xi_\mathcal{B}$ having the lowest $x$-coordinate, $\mathcal{C}_2 :=[c_{10},\ldots,c_{18}]$ the median, and $\mathcal{C}_3 :=[c_{19},\ldots,c_{27}]$ the highest. The centers are sorted lexicographically in every $\mathcal{C}_i$. Refer to Fig.~\ref{fig:radial} for an example.

The strength of our 26 height directions is that directional height filtrations discretize the Persistent Homology Transform~\cite{turner2014persistent}, which is injective on a broad class of shapes; hence, in principle, the family of height persistence diagrams determines the underlying shape with high accuracy, as shown later empirically. Our cube‑symmetric set of 26 directions provides an efficient spherical sampling that harmonizes with cubical complexes, producing salient and stable topological events that differ across object categories. By the stability of persistent homology, these diagrams are robust to noise. Complementing height with a set of $27$ radial filtrations, from carefully placed centers, and the following four filtrations, injects information about interior organization, yielding consistent gains. This is corroborated by the high accuracy numbers obtained for shape classification (see Sec.~\ref{sec:exp}).

\begin{wrapfigure}{r}{0.34\textwidth}
    \centering
    \vspace*{-30pt}
\includegraphics[scale=0.3]{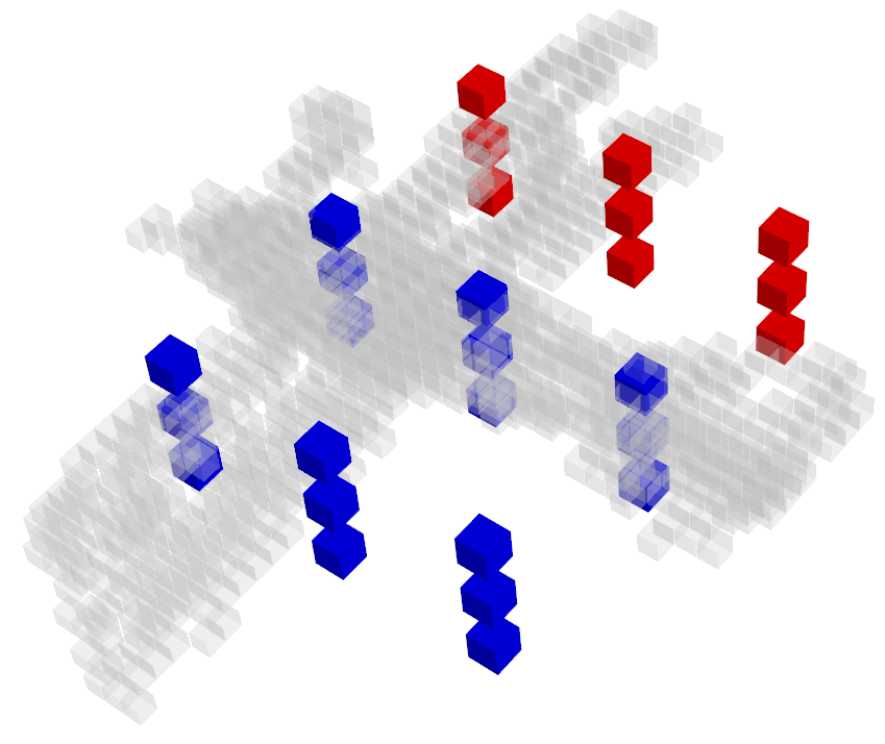}
  \caption{The $27$ radial centers are shown for an airplane 3D binary image. For ModelNet40 and ModelNet10, we have used the centers $c_1,\ldots,c_{18}$, as shown in blue.}
      \label{fig:radial}
      \vspace{-30pt}
\end{wrapfigure}

\textit{Density filtration.} Every voxel $p\in \mathcal{B}$ is assigned a grayscale value equal to the number of activated voxels within a ball centered at $p$ having radius $r$. We fixed $r$ to $1$ in our experiments.

\textit{Dilation filtration.} Every voxel $p\in \mathcal{B}$ is assigned a grayscale value equal to the smallest Manhattan distance to an activated voxel in $\mathcal{B}$. Consequently, active voxels are assigned a $0$ grayscale value.

\textit{Erosion filtration.} It does the opposite of the dilation filtration. The dilation filtration is applied to the binary image $\mathcal{B}'$, obtained from $\mathcal{B}$ by changing activated voxels to deactivated and deactivated ones to activated. Deactivated voxels are assigned a $0$ grayscale value.

\textit{Signed distance filtration.} For every activated voxel $p\in \mathcal{B}$, its grayscale value is the minimum Manhattan distance between $p$ and any deactivated voxel in $\mathcal{B}$ minus $1$. For every deactivated voxel, its grayscale value is the negative of the minimum Manhattan distance between $p$ and any activated voxel in $\mathcal{B}$.

TDA~\cite{chazal2021introduction, wasserman2018topological} can extract topological information and geometric patterns from datasets using algebraic topology. Persistent homology~\cite{edelsbrunner2002topological, zomorodian2004computing} is a popular tool in TDA, applied to obtain different kinds of topological feature vectors of point clouds. It helps to understand the shape of a point cloud by tracking the birth and death of various topological features (different from feature vectors), such as connected components, holes, and higher-dimensional voids that persist at different scales during an iterative process known as \textit{filtration} (distinct from the six types of filtration mentioned above). A series of nested geometric structures is obtained at various scales during filtration.  The topological features that persist (survive) across several iterations of a filtration, inside various sequences of such nested structures, can be used as topological descriptors to compute different topological feature vectors of a point cloud. For a comprehensive overview on persistent homology and various filtration techniques, we urge interested readers to refer to~\cite{chazal2021introduction, edelsbrunner2002topological, wasserman2018topological,zomorodian2004computing}. We use cubical homology and persistence meant for extracting topological information from cubical complexes.

\textbf{Cubical homology and persistence}.~\cite{kaczynski2006computational, wagner2011efficient} A finite \textit{cubical complex} in $3$-space is a union of points, line segments, squares, 3D cubes, aligned on the grid $\mathbf{Z}^3$.
We leverage \textit{cubical homology}, a variant of persistent homology meant for cubical complexes, to obtain topological feature vectors of grayscale 3D images, which are obtained from 3D binary images, constructed for every point cloud. 
 Any grayscale 3D image can be perceived as a cubical complex $K$, where every voxel (a pixel in 3D) is a cube with an intensity value.  The voxels, square faces of voxels, edges, and vertices are $3, 2, 1, 0$-cube, respectively. Hence, cubical homology can be applied directly to 3D grayscale images because of their natural grid-like structures. During filtration, the voxels are added in order of increasing intensity, forming a sequence of nested cubical subcomplexes. Starting with the lowest voxel intensity, all voxels with intensity at most $t$ are added to the current cubical complex along with their faces, edges, and vertices, where $t$ is the current voxel intensity being considered. A cubical complex obtained at step $i$ is a subcomplex of the complex obtained at step $i+1$. Thus, we get a sequence of nested subcomplexes $K_0 \subseteq K_1 \subseteq \ldots \subseteq K_m$.  Every $K_i$ is called a \textit{sublevel set} of $K$, the cubical complex built from a given 3D image, since $K_i\subseteq K$.
  As voxels are added, the topological features, connected components of voxels (homology group $H_0$), tunnels or loops (homology group $H_1$), and enclosed cavities (homology group $H_2$), take birth or die. Every birth and death of a feature introduces a new birth-death pair in the  \textit{cubical persistence}, a multiset of points in $\mathbf{R} \times (\mathbf{R} \cup \{+\infty\})$, where every pair $(b,d)$ in the multiset denotes the birth of a topological feature at time $b$ and its death at time $d$. Long-surviving features are likely the significant features that can be used in classification tasks. Persistence, represented using a 2D scatter plot, is known as a \textit{persistence diagram} (refer to Fig. \ref{fig:intro}). In a persistence diagram, every birth-death pair corresponds to a point in the diagram. The cubical persistence of a cubical complex is its \textit{topological signature}. Next, we discuss the topological vectorization methods used here.

\textbf{Persistent entropy.}~\cite{chintakunta2015entropy} Given a cubical persistence, $X=\{(b_i,d_i)\}$, its persistent entropy, denoted by $\rho(X)$, is a real number defined by as,
$\rho(X) = -\sum_i p_i \log (p_i) $, where $p_i = \frac{d_i - b_i}{\ell(X)}$, and $\ell(X) = \sum_i (d_i - b_i)$. Having its roots in information theory, it gives an intuitive sense of disorder or complexity in the topological structure. We note that $3$ real numbers are obtained for the $3$ homology groups, $H_0,H_1, \text{ and } H_2$.  

\textbf{Amplitude.} Introduced in~\cite{garin2019topological}, \textit{amplitude} of a cubical persistence is defined as its distance to the empty persistence (devoid of birth-death pairs). It is used to compare two cubical persistences, obtained from two different 3D grayscale images. ~\taco~ uses five types of amplitudes with varied parameters. Out of the five, two are metric-based (the Wasserstein and Bottleneck distances), and the remaining three are kernel-based (Betti curve, persistence landscape, and heat). For the Betti curve, and persistence landscape, the diagrams are sampled using $100$ filtration values, whereas for the heat kernel, $20$ are used. Let $X=\{(b_i,d_i)\}$ be a cubical persistence. For an insight into the different kinds of amplitudes used, we recommend that the reader refer to the Appendix.

\textit{$p$-Wasserstein}~\cite{tauzin2021giotto}. The \textit{half-lifetime} of a pair $(b_i,d_i)\in X$ is defined as $\frac{d_i-b_i}{2}$. The Wasserstein amplitude of order $p$, denoted by $W(X,p)$, is defined as the $L_p$ norm of the vector of half-lifetimes of the birth-death pairs in $X$. Hence, $W(X,p) = ( \sum_{i} \left( \frac{d_i - b_i}{2}\right)^p )^{1/p}.$ For~\taco~, we have used $p=1,2$. We obtain $6$ real numbers for this metric, since there are $3$ homology groups and $2$ values of $p$.

\textit{Bottleneck}~\cite{tauzin2021giotto}. The Bottleneck amplitude is denoted by  $B(X) = W(X,\infty)$. We obtain $3$ real numbers for this metric due to the three homology groups.

\textit{Betti curve.~\cite{tauzin2021giotto}} The Betti curve of $X$ is the function $B_C:\mathbf{R} \rightarrow \mathbf{N}$, such that $B_C(s)$ gives the number of birth-pairs in $X$ that contains $s$ when every pair $(b_i,d_i)$ in $X$ is treated as an interval. Two amplitudes are obtained using the $L_1$ and $L_2$ norms. We obtain $6$ real numbers for this metric, since there are three homology groups and two norms.

\textit{Landscape}~\cite{bubenik2017persistence, bubenik2015statistical}. For a birth-death pair  $(b_i,d_i) \in X$, let $f_{(b_i,d_i)} : \mathbf{R} \rightarrow [0,\infty]$, be a piecewise linear function given in Eq. \ref{eq:landscape}. 

\begin{wrapfigure}{r}{0.5\textwidth}

	\begin{equation}
		f_{(b_i,d_i)} (x) =  \begin{cases}
			0 & \text{if } x \notin (b_i,d_i) \\
			x-b_i & \text{if } x \in (b_i,\frac{b_i+d_i}{2}] \\
			-x+d_i & \text{if } x \in  (\frac{b_i+d_i}{2}, d_i)
		\end{cases}
		\label{eq:landscape}
	\end{equation}

\end{wrapfigure}

The \textit{persistence landscape} of $X$ is the sequence of functions $\lambda_k : \mathbf{R} \rightarrow [0, \infty], k =1, 2, 3, \ldots$ where $\lambda_k (x)$ is the $k$-th largest value of $\{f_{(b_i,d_i)}(x)\}_i$. Further, $\lambda_k(x)$ is set to $0$ if the $k$-th largest value does not exist. 
The parameter $k$ is called the \textit{layer}. For~\taco~, we have used $k=1,2$. Four amplitudes are obtained using $L_1$ and $L_2$ norms for both the values of $k$. We get $12$ real numbers for this metric, since there are three homology groups, two norms, and two distinct values of $k$.

\textit{Heat kernel}~\cite{reininghaus2015stable}. Gaussians of standard deviation $\sigma$ are placed over every point in $X$ and a negative Gaussian of $\sigma$ on the mirror point across the diagonal line in the persistence diagram. Thus, a real-valued function is obtained on $\mathbf{R}^2$. For \taco~, we have used $\sigma=0.15$. We get $6$ real numbers for this metric, since there are three homology groups and two norms, $L_1,L_2$.

Hence, for a given 3D grayscale image, obtained by using a filtration, we get a feature vector of length $3 + 33 = 36$, wherein $3$ numbers are obtained using persistent entropy and the remaining $33$ using amplitude. 

\textbf{Feature selection and generation.} Refer to Fig.~\ref{fig:pipeline} in Appendix for an illustration. Let $P$ be a point cloud describing some 3D object. We convert $P$ into a voxelized 3D binary image $\mathcal{B}$ such that every active voxel contains at least one point from $P$. The volume of $\mathcal{B}$ is roughly equal to that of the axis-parallel bounding box of $P$. We run $57$ filtrations on $\mathcal{B}$ yielding a set of $57$ grayscale images. Out of $57$ filtrations, there are $26$ height filtrations for the $26$ direction vectors in $\{\{0,1,-1\}^3 \}\setminus \{ (0,0,0) \}$; $27$ radial filtrations for the $27$ centers $c_1,\ldots,c_{27}$, as described in Section~\ref{sec:tda}; one each for the four types: density, dilation, erosion, and signed distance. As explained before, we extract $36$ features from a grayscale image. Hence, due to the $57$ filtrations used, which resulted in $57$ binary images, the length of the final feature vector for $P$ is $57 \cdot 36 = 2052$. However, our experiments found that depending on the dataset, we must discard some of the radial filtrations from the initial $27$ centers, as shown in Fig. \ref{fig:radial} to achieve the highest possible accuracy.

\usetikzlibrary{positioning}


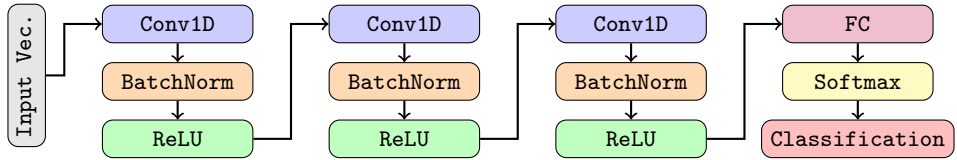
\begin{figure}[h!]
\centering
\begin{tikzpicture}[
    layer/.style={rectangle, draw, minimum height=0.5cm, minimum width=2cm, text centered, font=\sffamily\footnotesize},
    conv/.style={layer, fill=blue!20, rounded corners},
    bn/.style={layer, fill=orange!30, rounded corners},
    relu/.style={layer, fill=green!25, rounded corners},
    fc/.style={layer, fill=purple!25, rounded corners},
    softmax/.style={layer, fill=yellow!30, rounded corners},
    input/.style={layer, fill=gray!20, rounded corners},
    output/.style={layer, fill=red!25, rounded corners},
    arrow/.style={->, thick},
]

\node[rounded corners, rotate=90, rectangle, draw, fill=gray!20, font=\sffamily\footnotesize] (input) at (0,0) {\texttt{Input Vec.}};

\node[conv, right=of input, yshift=-0.27cm]    (conv1)  {\texttt{Conv1D}};
\node[bn,   below=0.25cm of conv1]   (bn1)    {\texttt{BatchNorm}};
\node[relu, below=0.25cm of bn1]      (relu1)  {\texttt{ReLU}};

\node[conv, right=1cm of conv1]    (conv2)  {\texttt{Conv1D}};
\node[bn,   below=0.25cm of conv2]          (bn2)    {\texttt{BatchNorm}};
\node[relu, below=0.25cm of bn2]            (relu2)  {\texttt{ReLU}};

\node[conv, right=1cm of conv2]    (conv3)  {\texttt{Conv1D}};
\node[bn,   below=0.25cm of conv3]          (bn3)    {\texttt{BatchNorm}};
\node[relu, below=0.25cm of bn3]            (relu3)  {\texttt{ReLU}};

\node[fc, right=1cm of conv3]       (fc1)    {\texttt{FC}};
\node[softmax, below=0.25cm of fc1]          (softmax){\texttt{Softmax}};
\node[output, below=0.25cm of softmax]       (output) {\texttt{Classification}};

\draw[arrow] (input) -- ++(0.5,0) |- (conv1);
\draw[arrow] (conv1) -- (bn1);
\draw[arrow] (bn1) -- (relu1);
\draw[arrow] (relu1.east) -- ++(0.5,0) |- (conv2);
\draw[arrow] (conv2) -- (bn2);
\draw[arrow] (bn2) -- (relu2);
\draw[arrow] (relu2.east) -- ++(0.5,0) |- (conv3);
\draw[arrow] (conv3) -- (bn3);
\draw[arrow] (bn3) -- (relu3);

\draw[arrow] (relu3.east) -- ++(0.5,0) |- (fc1);

\draw[arrow] (fc1) -- (softmax);
\draw[arrow] (softmax) -- (output);

\end{tikzpicture}
\caption{The architecture of the 1D CNN used in \taco~. Convolutional layers process the input feature vector before final classification via fully connected (FC) and softmax layers.}
\label{fig:nn_architecture}
\end{figure}


\begin{wraptable}{r}{0.34\textwidth}
    \vspace{-14pt}
    \centering
    \caption{Parameters and values}
    \addtolength{\tabcolsep}{-0.4em}

    \begin{tabular}{cc}
    \toprule
        Parameters & Values\\ \midrule
        Max. training epoch & 1000\\
        Loss stop threshold & 0.005\\
        Learning rate & 0.001\\
        Minibatch size & 128\\
        Optimizer & Adam\\ 
        Voxel size & 0.05\\
        \bottomrule
    \end{tabular}
    \label{tab:params}
    \vspace{-10pt}
\end{wraptable}
\textbf{Classification using a 1D CNN.} We use a lightweight 1D CNN deep neural network to classify the features extracted from the point clouds of the objects. In recent years, 1D CNN has been used extensively for such feature and sequence classification with high success~\cite{kiranyaz20211d}. Our network has three 1D CNN layers, each followed by batch normalization and ReLU layers. The three CNN layers after the input have filter sizes of $3, 5, $ and $7$ respectively, whereas the number of filters in the first two layers is $128$ and $64$, respectively. In the third CNN layer, the filters are set to the class count for ModelNet40 and ModelNet10 datasets and $32$ for the VesselMNIST3D and AdrenalMNIST3D datasets.
After the three consecutive 1D CNN layers, we have a fully connected layer of size equal to the number of classes. Next, the classification is done by applying a softmax function on the outputs of the fully connected layer. Refer to Fig.~\ref{fig:nn_architecture}. The time taken for classification is $\mathcal{O}(n + v^3 + v/\rho^3 )$, where $n$ is the size of the point cloud, $v$ is the number of voxels in $\mathcal{B}$, and $\rho$ is the voxel-size used. See Appendix for a proof.

\section{Experiments}
\label{sec:exp}

\noindent
\textbf{Settings. }
We have used six datasets to validate the efficacy of \taco~and they are ModelNet10/40 (by far the most popular benchmark for this problem), OmniObject3D (a real-world dataset consisting of $190$ classes), ScanObjectNN (a real-world noisy dataset with $15$ classes), and two real-world binary medical datasets for further generalizability, namely VesselMNIST3D and AdrenalMNIST3D. More details about these datasets are provided in the Appendix (Sec. \ref{sec:datasets}). Our proposed \taco~has achieved SOTA accuracy in five out of these six test datasets.

We have implemented all TDA-related portions of \taco~ in Python using the \texttt{giotto-tda} package~\cite{tauzin2021giotto}. The experiments were run on a machine equipped with an Intel \texttt{i9-12900K} processor, $32$-GB of main memory, and a NVIDIA RTX 3060 GPU. For all the datasets, we evaluate the performance on two main metrics: overall accuracy (OA - average accuracy \% across all test cases) and mean class accuracy  (mAcc - mean accuracy \% across all classes). These are the most common evaluation metrics in object classification. For VesselMNIST3D, we also present the F1-score metric due to imbalance.

For ModelNet40, we have found the vector length $1728$ to be the optimum in terms of OA. The exact number of filtrations that corresponds to this length is $48$, out of which $26$ are height, one each from the four types: density, dilation, erosion, and signed distance, and $18$ are radial corresponding to the $18$ centers $c_1,\ldots,c_{18}$, as described in Section~\ref{sec:tda} and shown in Fig. \ref{fig:radial}. Therefore, all the results presented below are with $1728$-length feature vectors. We have used the same vector length input for ModelNet10 as well. The optimum length for VesselMNIST is $1152$ using the two centers $c_1,c_2$, and for AdrenalMNIST it is $1584$ using $c_1,\ldots,c_{14}$. Other relevant parameters and their values for TDA-related experiments are mentioned in Section~\ref{sec:tda}.

We used MATLAB to implement the 1D CNN. We stopped our training early if the training loss had reached $0.005$. Each configuration has been trained $5$ times, and the average results are presented in the paper unless specified otherwise. The learnable parameters for ModelNet40, ModelNet10, VesselMNIST, and AdrenalMNIST are $0.72$M, $0.71$M, $0.50$M, $0.66$M, respectively. The parameters used in our experiments and their values are listed in Table \ref{tab:params}.

\subsection{Results}
\noindent
\textbf{ModelNet10/40. }
First, we present the results of testing \taco~ on ModelNet40 and 10 datasets. To begin with, we first illustrate the empirical reason behind choosing 18 radial filtration centers along with DEDS and height filters for the ModelNet40 dataset. This result is presented in Fig. \ref{fig:MN40_length_pr}(a). As can be seen, with feature vector length $1728$, i.e., 18 radial filtration filters, the OA is the highest. Although with the different other feature lengths, the OA is close, but lower than the one with length = $1728$. Next, we have tested \taco~ with different voxel sizes to create the 3D binary image from the given point cloud. We have noticed that with a voxel size $0.05$, the OA is the highest. With $0.03$ and $0.07$, the accuracy values decrease to 98.61 (OA), 96.96 (mAcc), and 98.30 (OA) and  96.16 (mAcc), respectively.

Next, the benchmark results for both 40- and 10-class variants of the ModelNet dataset are presented in Tables \ref{tab:modelnet40_results} and \ref{tab:modelnet10_results} (in the Appendix), respectively. These results prove that our proposed \taco~ framework achieves state-of-the-art OA and mAcc accuracies for both these datasets. Notably, \taco~ achieves $1.68\%$ higher OA than RotationNet, the current highest OA-achieving method on the ModelNet40 dataset. Further, \taco~ comprehensively outperformed the recent transformer-based models, e.g., PointMamba~\cite{liang2024pointmamba} and PointGPT~\cite{chen2023pointgpt}, among others. Our macro-averaged precision-recall curve (Fig. \ref{fig:MN40_length_pr}(b)) stays tightly clustered near $(1,1)$, showcasing near-perfect precision and recall across every class. This level of consistency, even on minority classes, sets a new bar for robust, balanced multi-class performance. Fig. \ref{fig:MN40_cm} (refer to Appendix) shows the confusion matrix found with the best saved model.

Similarly, \taco~ outperforms RotationNet in OA on the ModelNet10 dataset. Given the lower number of classes available, it was expected that the proposed \taco~ framework would achieve higher accuracies in ModelNet10 than in ModelNet40. Not only was \taco~ successful in meeting that expectation, it yielded $99.52\%$ OA and mAcc values. Most importantly, to the best of our knowledge, ours is the first approach to push the classification accuracy beyond $99\%$ on both ModelNet40 and ModelNet10. Altogether, these results make this research work groundbreaking.

\noindent
\textbf{Real-world Datasets. }
To further validate real-world applicability, we evaluate \taco~ on OmniObject3D~\cite{wu2023omniobject3d}, a challenging benchmark featuring thousands of everyday objects captured under realistic conditions. Despite its complexity and large class diversity, \taco~ delivers the highest accuracy of $58.90\%$, decisively outperforming heavyweight baselines including CurveNet, PointNet, PointNet++, and PCT: Point cloud transformer (see Table~\ref{tab:omniobject3d}). We next evaluate \taco~ on ScanObjectNN (OBJ\_BG variant), a notoriously challenging real-world benchmark characterized by heavy occlusions and background clutter-conditions where many point-based methods struggle~\cite{uy2019revisiting}. While not setting a new record, \taco~ achieves an impressive $93.94\%$ overall accuracy, surpassing widely adopted baselines such as PointNet~\cite{qi2017pointnet} (73.3), SpiderCNN~\cite{xu2018spidercnn} (77.1), PointNet++~\cite{qi2017pointnet++} (82.3), DGCNN~\cite{wang2019dynamic} (82.8), and even edging out advanced models like GDANet~\cite{xu2021learning} (87.0), PointBERT~\cite{yu2022point} (87.43), and PointGPT~\cite{chen2023pointgpt} (93.39). This result underscores \taco~'s ability to remain highly competitive against transformer and graph-based architectures in highly cluttered, real-world scenarios. 

\noindent
\textbf{Real-world Medical Data. }For VesselMNIST~\cite{yang2020intra}, we compared against some of the current SOTA benchmarks for this dataset as shown in Table \ref{tab:vessel3d}. When compared against the benchmarks presented in \cite{yang2020intra}, \taco~ performed better in terms of both F1 and mAcc. For example, the previous highest F1-score of $0.90$ for VesselMNIST was achieved by PointNet++ and PointCNN, whereas our mean F1-score is $0.94$ - an improvement of $4.44\%$. Similarly, for the mAcc metric, our average result is $95.28\%$, whereas the prior best was $93.52$ achieved by PointNet++ - an improvement of $1.88\%$.

For the AdrenalMNIST dataset, we used the benchmark provided in \cite{yang2023medmnist} as our baseline. The comparison results are presented in Table \ref{tab:adrenal}. The authors in \cite{yang2023medmnist} have used different variations of ResNet along with medical image-specific variants such as ACS~\cite{yang2021reinventing}. Our proposed \taco~ outperformed all these benchmarks in the OA metric, as shown in the table. Notably, for the VesselMNIST3D dataset, the difference in OA between ours and the current SOTA achieved using ResNet-18 + ACS is $4.58\%$.

\textbf{Resiliency Against Corrupted Test Data.} Noise resiliency in ~\taco~ is achieved because topological features, extracted via persistent homology, capture the global shape characteristics of objects rather than relying on exact point positions. These features remain stable under small perturbations or noise, as persistent homology emphasizes long-lived topological structures while ignoring short-lived, noise-induced artifacts. Consequently, ~\taco~ maintains high accuracy in most cases even when point clouds are corrupted, as discussed below. To test the robustness of \taco~, we have used the standardized approach of ModelNet40-C~\cite{sun2022benchmarking}, where the test set of ModelNet40 is perturbed by adding various common types of corruptions. Note that the `uniform downsampling' perturbation was not part of ModelNet40-C.

Two main differences between our implementation and ModelNet40-C are 1) we use $2048 \times 3$ size point clouds, and 2) we do not apply any normalization after the perturbation is incorporated. The results for this study are presented in Table \ref{tab:modelnet-C}. 
The first row presents the test results found by the model with the highest test accuracy for clean ModelNet40 without any corruption. We use this saved model for inference on the corrupted test dataset and report the results in Table \ref{tab:modelnet-C}. Two severity levels of data corruption are chosen from \cite{sun2022benchmarking}: 1 and 5, named Low and High, respectively, in our paper. For `uniform downsampling', we have removed $10$ and $30\%$ random points uniformly for low and high severity, respectively.

We see that \taco~ is very robust against low-level perturbations - always achieves $\ge 94\%$ OA except for `impulse', where the OA falls to $52.88\%$. As expected with high-level perturbations, \taco~ shows resiliency. In case of `impulse', the OA more than halves, but for the others, it performs reasonably well - the average OA being $68.65\%$. If the underlying shape changes significantly, then topological features become very different from clean class signatures, and \taco~ struggles to recognize corrupted point clouds. Under the high‑severity corruption, augmenting the training set with corrupted copies of randomly selected 20\% of training instances lifts overall accuracy to 96.11\% in the case of rotation, for example, substantially outperforming the non‑augmented model (49.39\%).
This demonstrates that a task‑aligned augmentation confers significant robustness to extreme pose variations without altering the core architecture.

\noindent
\textbf{Shape Retrieval. }Our method achieves a new state-of-the-art retrieval performance with an mAP of $99.33$ on ModelNet40, surpassing the strongest baseline Latformer ($97.4$) and all prior approaches (see summary results in Table \ref{tab:modelnet40_retrieval} in the Appendix). Beyond accuracy, the framework demonstrates strong generalizability, showing its effectiveness not only for classification but also for challenging tasks such as shape retrieval.

\noindent
\textbf{Note on efficiency. }With voxel size $\rho=0.05$ (our optimal configuration), the throughput of the TDA feature generation pipeline is $5.3$ point clouds/second, when the feature vector length was set to $1728$ for the ModelNet40 dataset. This number increased to $8.2$ and decreased to $1.4$ when $\rho$ was increased to $0.07$ or decreased to $0.03$, respectively. However, as mentioned earlier, the test accuracy decreases in both cases. However, on the bright side, due to the lightweight 1D CNN of \taco~, the training time is short ($2.50$ mins.) and class prediction is lightning-fast -- achieving a throughput rate of $16,454$ point clouds/second. Furthermore, Table \ref{tab:params_only} demonstrates that our proposed \taco~model achieves state-of-the-art accuracy with the very few parameters (0.72M), highlighting its superior efficiency compared to prior methods.

\begin{wraptable}{r}{0.4\textwidth}
\vspace{-14pt}

    \centering
    \caption{Ablation study on ModelNet40.}
    \addtolength{\tabcolsep}{-0.4em}

    \begin{tabular}{p{0.25cm}p{2.5cm}cc}
    \toprule
    &Variant & OA  & mAcc \\
    \midrule
    \multirow{4}{4em}{\rotatebox{90}{\underline{features}}} & DEDS only  & 96.52 & 93.76\\
       & H only  & 98.29 & 96.38\\
       & DEDS + H only  & 98.82 & 97.33\\
       & Entropy only & 96.16 & 92.79\\
       \midrule
       \multirow{2}{*}{\rotatebox{90}{\underline{net}}} & First two Conv1D & 98.18 & 95.93\\
       & First Conv1D &94.76 & 90.70\\
    \bottomrule
    \end{tabular}
    \label{tab:ablation}
    \vspace{-10pt}
\end{wraptable}

\textbf{Ablation Study. }
Our approach to studying the effect of ablation is two-fold. First, using our proposed network (Fig. \ref{fig:nn_architecture}), we test different sets of topological features extracted from a point cloud, i.e., using density, erosion, dilation, and signed distance filtration (DEDS) only, height (H) filtration only, and finally DEDS + H only (i.e., without any radial filtration). The effect of using different radial filtration along with DEDS and H together is already illustrated in Fig. \ref{fig:MN40_length_pr}(a) and discussed earlier. In the next set of ablation studies, we use the best topological features for ModelNet40, i.e., a vector length of $1728$, while testing the effect of deleting one CNN layer at a time. The results are listed in Table \ref{tab:ablation}. This study demonstrates that the features extracted via the H filtration are the most effective, yielding over 
98\% OA, which outperforms all baselines while being 3.5x faster than the whole pipeline. Similarly, when just entropy is used (without amplitude), vector generation is 2.2x faster while maintaining $96.16\%$ OA.
This shows TACO-Net is not only accurate but also tunable for resource-constrained scenarios. The finding further supports our rationale for employing 
26 directional vectors, as outlined in Sec.~\ref{sec:tda}. Incorporating the DEDS features provides a slight improvement, but the gain is marginal. In contrast, removing two CNN layers leads to a moderate decrease in 
4.29\% in accuracy. To highlight the effectiveness of our 1D CNN for feature vector classification, we replace it with a heavier 2.2M-parameter transformer that incorporates feature and positional embeddings, mixed self-attention, and a fully connected classifier. Despite its complexity, this transformer yields only $62.20\%$ accuracy on ModelNet40, underscoring the superiority of our lightweight CNN design. On the other hand, XGBoost, a non-deep learning method also yielded substantially lower OA of 81.35\% (see Table~\ref{tab:xgboost_comp}). 
Taken together, these results provide strong justification for our feature selection and network design choices.

\section{Conclusion and Future Work}

3D object classification is an important task for autonomous systems. Furthermore, such classification can become standard in automated diagnosis with 3D medical imaging. Computer vision researchers have made significant advancements in this topic using various deep learning techniques in recent years. However, there are still some challenges to address and open directions to explore. To this end, we have proposed a novel framework, named \taco~, for 3D object classification. Our proposed approach takes a point cloud of the object as input, converts it into a voxelized 3D binary image, extracts topological signatures from it through various filtration techniques, and finally learns these features using a lightweight 1D CNN. Results show that our proposed technique achieves near-perfect overall accuracy in popular 3D object classification benchmark datasets, namely ModelNet40 and ModelNet10, while outperforming the current SOTA for these. Further, when tested on two 3D medical datasets consisting of brain MRA and abdominal CT scan data, \taco~, outperforms all the benchmarks provided in the literature, showcasing its strong generalizable qualities. To enhance real-time performance, future work will focus on exploiting GPU parallelism to increase throughput for the topological feature generation pipeline.

\bibliographystyle{acm}
\bibliography{refs}

\begin{thebibliography}{10}

\bibitem{akcora2020bitcoinheist}
{\sc Akcora, C.~G., Li, Y., Gel, Y.~R., and Kantarcioglu, M.}
\newblock Bitcoinheist: Topological data analysis for ransomware prediction on
  the bitcoin blockchain.
\newblock In {\em Proceedings of the Twenty-Ninth International Joint
  Conference on Artificial Intelligence\/} (2020).

\bibitem{ben20183dmfv}
{\sc Ben-Shabat, Y., Lindenbaum, M., and Fischer, A.}
\newblock 3dmfv: Three-dimensional point cloud classification in real-time
  using convolutional neural networks.
\newblock {\em IEEE Robotics and Automation Letters 3}, 4 (2018), 3145--3152.

\bibitem{brock2016generative}
{\sc Brock, A., Lim, T., Ritchie, J.~M., and Weston, N.}
\newblock Generative and discriminative voxel modeling with convolutional
  neural networks.
\newblock {\em arXiv preprint arXiv:1608.04236\/} (2016).

\bibitem{bubenik2017persistence}
{\sc Bubenik, P., and D{\l}otko, P.}
\newblock A persistence landscapes toolbox for topological statistics.
\newblock {\em Journal of Symbolic Computation 78\/} (2017), 91--114.

\bibitem{bubenik2015statistical}
{\sc Bubenik, P., et~al.}
\newblock Statistical topological data analysis using persistence landscapes.
\newblock {\em J. Mach. Learn. Res. 16}, 1 (2015), 77--102.

\bibitem{bukkuri2021applications}
{\sc Bukkuri, A., Andor, N., and Darcy, I.~K.}
\newblock Applications of topological data analysis in oncology.
\newblock {\em Frontiers in artificial intelligence 4\/} (2021), 659037.

\bibitem{chazal2021introduction}
{\sc Chazal, F., and Michel, B.}
\newblock An introduction to topological data analysis: fundamental and
  practical aspects for data scientists.
\newblock {\em Frontiers in artificial intelligence 4\/} (2021), 667963.

\bibitem{chen2023pointgpt}
{\sc Chen, G., Wang, M., Yang, Y., Yu, K., Yuan, L., and Yue, Y.}
\newblock Pointgpt: Auto-regressively generative pre-training from point
  clouds.
\newblock {\em Advances in Neural Information Processing Systems 36\/} (2023),
  29667--29679.

\bibitem{chen2016xgboost}
{\sc Chen, T., and Guestrin, C.}
\newblock Xgboost: A scalable tree boosting system.
\newblock In {\em Proceedings of the 22nd acm sigkdd international conference
  on knowledge discovery and data mining\/} (2016), pp.~785--794.

\bibitem{chintakunta2015entropy}
{\sc Chintakunta, H., Gentimis, T., Gonzalez-Diaz, R., Jimenez, M.-J., and
  Krim, H.}
\newblock An entropy-based persistence barcode.
\newblock {\em Pattern Recognition 48}, 2 (2015), 391--401.

\bibitem{edelsbrunner2002topological}
{\sc Edelsbrunner, Letscher, and Zomorodian}.
\newblock Topological persistence and simplification.
\newblock {\em Discrete \& computational geometry 28\/} (2002), 511--533.

\bibitem{fabbri20082d}
{\sc Fabbri, R., Costa, L. D.~F., Torelli, J.~C., and Bruno, O.~M.}
\newblock 2d euclidean distance transform algorithms: A comparative survey.
\newblock {\em ACM Computing Surveys (CSUR) 40}, 1 (2008), 1--44.

\bibitem{feng2019hypergraph}
{\sc Feng, Y., You, H., Zhang, Z., Ji, R., and Gao, Y.}
\newblock Hypergraph neural networks.
\newblock In {\em Proceedings of the AAAI conference on artificial
  intelligence\/} (2019), vol.~33, pp.~3558--3565.

\bibitem{feng2018gvcnn}
{\sc Feng, Y., Zhang, Z., Zhao, X., Ji, R., and Gao, Y.}
\newblock Gvcnn: Group-view convolutional neural networks for 3d shape
  recognition.
\newblock In {\em Proceedings of the IEEE conference on computer vision and
  pattern recognition\/} (2018), pp.~264--272.

\bibitem{feurer2015efficient}
{\sc Feurer, M., Klein, A., Eggensperger, K., Springenberg, J., Blum, M., and
  Hutter, F.}
\newblock Efficient and robust automated machine learning.
\newblock {\em Advances in neural information processing systems 28\/} (2015).

\bibitem{garin2019topological}
{\sc Garin, A., and Tauzin, G.}
\newblock A topological "reading" lesson: Classification of mnist using tda.
\newblock In {\em 2019 18th IEEE international conference on machine learning
  and applications (ICMLA)\/} (2019), IEEE, pp.~1551--1556.

\bibitem{goyal2021revisiting}
{\sc Goyal, A., Law, H., Liu, B., Newell, A., and Deng, J.}
\newblock Revisiting point cloud shape classification with a simple and
  effective baseline.
\newblock In {\em International conference on machine learning\/} (2021), PMLR,
  pp.~3809--3820.

\bibitem{guo2021pct}
{\sc Guo, M.-H., Cai, J.-X., Liu, Z.-N., Mu, T.-J., Martin, R.~R., and Hu,
  S.-M.}
\newblock Pct: Point cloud transformer.
\newblock {\em Computational visual media 7}, 2 (2021), 187--199.

\bibitem{hamdi2021mvtn}
{\sc Hamdi, A., Giancola, S., and Ghanem, B.}
\newblock Mvtn: Multi-view transformation network for 3d shape recognition.
\newblock In {\em Proceedings of the IEEE/CVF international conference on
  computer vision\/} (2021), pp.~1--11.

\bibitem{he2024latformer}
{\sc He, X., Cheng, S., Liang, D., Bai, S., Wang, X., and Zhu, Y.}
\newblock Latformer: locality-aware point-view fusion transformer for 3d shape
  recognition.
\newblock {\em Pattern Recognition 151\/} (2024), 110413.

\bibitem{jin2019auto}
{\sc Jin, H., Song, Q., and Hu, X.}
\newblock Auto-keras: An efficient neural architecture search system.
\newblock In {\em Proceedings of the 25th ACM SIGKDD international conference
  on knowledge discovery \& data mining\/} (2019), pp.~1946--1956.

\bibitem{kaczynski2006computational}
{\sc Kaczynski, T., Mischaikow, K., and Mrozek, M.}
\newblock {\em Computational homology}, vol.~157.
\newblock Springer Science \& Business Media, 2006.

\bibitem{kanezaki2019rotationnet}
{\sc Kanezaki, A., Matsushita, Y., and Nishida, Y.}
\newblock Rotationnet for joint object categorization and unsupervised pose
  estimation from multi-view images.
\newblock {\em IEEE transactions on pattern analysis and machine intelligence
  43}, 1 (2019), 269--283.

\bibitem{khan2019unsupervised}
{\sc Khan, S.~H., Guo, Y., Hayat, M., and Barnes, N.}
\newblock Unsupervised primitive discovery for improved 3d generative modeling.
\newblock In {\em Proceedings of the IEEE/CVF Conference on Computer Vision and
  Pattern Recognition\/} (2019), pp.~9739--9748.

\bibitem{kiranyaz20211d}
{\sc Kiranyaz, S., Avci, O., Abdeljaber, O., Ince, T., Gabbouj, M., and Inman,
  D.~J.}
\newblock 1d convolutional neural networks and applications: A survey.
\newblock {\em Mechanical systems and signal processing 151\/} (2021), 107398.

\bibitem{klokov2017escape}
{\sc Klokov, R., and Lempitsky, V.}
\newblock Escape from cells: Deep kd-networks for the recognition of 3d point
  cloud models.
\newblock In {\em Proceedings of the IEEE international conference on computer
  vision\/} (2017), pp.~863--872.

\bibitem{komarichev2019cnn}
{\sc Komarichev, A., Zhong, Z., and Hua, J.}
\newblock A-cnn: Annularly convolutional neural networks on point clouds.
\newblock In {\em Proceedings of the IEEE/CVF conference on computer vision and
  pattern recognition\/} (2019), pp.~7421--7430.

\bibitem{li2018so}
{\sc Li, J., Chen, B.~M., and Lee, G.~H.}
\newblock So-net: Self-organizing network for point cloud analysis.
\newblock In {\em Proceedings of the IEEE conference on computer vision and
  pattern recognition\/} (2018), pp.~9397--9406.

\bibitem{li2018pointcnn}
{\sc Li, Y., Bu, R., Sun, M., Wu, W., Di, X., and Chen, B.}
\newblock Pointcnn: Convolution on x-transformed points.
\newblock {\em Advances in neural information processing systems 31\/} (2018).

\bibitem{liang2024pointmamba}
{\sc Liang, D., Zhou, X., Xu, W., Zhu, X., Zou, Z., Ye, X., Tan, X., and Bai,
  X.}
\newblock Pointmamba: A simple state space model for point cloud analysis.
\newblock {\em Advances in neural information processing systems 37\/} (2024),
  32653--32677.

\bibitem{liu2019point2sequence}
{\sc Liu, X., Han, Z., Liu, Y.-S., and Zwicker, M.}
\newblock Point2sequence: Learning the shape representation of 3d point clouds
  with an attention-based sequence to sequence network.
\newblock In {\em Proceedings of the AAAI conference on artificial
  intelligence\/} (2019), vol.~33, pp.~8778--8785.

\bibitem{liu2019densepoint}
{\sc Liu, Y., Fan, B., Meng, G., Lu, J., Xiang, S., and Pan, C.}
\newblock Densepoint: Learning densely contextual representation for efficient
  point cloud processing.
\newblock In {\em Proceedings of the IEEE/CVF international conference on
  computer vision\/} (2019), pp.~5239--5248.

\bibitem{liu2019relation}
{\sc Liu, Y., Fan, B., Xiang, S., and Pan, C.}
\newblock Relation-shape convolutional neural network for point cloud analysis.
\newblock In {\em Proceedings of the IEEE/CVF conference on computer vision and
  pattern recognition\/} (2019), pp.~8895--8904.

\bibitem{ma2018learning}
{\sc Ma, C., Guo, Y., Yang, J., and An, W.}
\newblock Learning multi-view representation with lstm for 3-d shape
  recognition and retrieval.
\newblock {\em IEEE Transactions on Multimedia 21}, 5 (2018), 1169--1182.

\bibitem{ma2022rethinking}
{\sc Ma, X., Qin, C., You, H., Ran, H., and Fu, Y.}
\newblock Rethinking network design and local geometry in point cloud: A simple
  residual mlp framework.
\newblock {\em arXiv preprint arXiv:2202.07123\/} (2022).

\bibitem{maturana2015voxnet}
{\sc Maturana, D., and Scherer, S.}
\newblock Voxnet: A 3d convolutional neural network for real-time object
  recognition.
\newblock In {\em 2015 IEEE/RSJ international conference on intelligent robots
  and systems (IROS)\/} (2015), IEEE, pp.~922--928.

\bibitem{mohammadi2021pointview}
{\sc Mohammadi, S.~S., Wang, Y., and Del~Bue, A.}
\newblock Pointview-gcn: 3d shape classification with multi-view point clouds.
\newblock In {\em 2021 IEEE International Conference on Image Processing
  (ICIP)\/} (2021), IEEE, pp.~3103--3107.

\bibitem{qi2017pointnet}
{\sc Qi, C.~R., Su, H., Mo, K., and Guibas, L.~J.}
\newblock Pointnet: Deep learning on point sets for 3d classification and
  segmentation.
\newblock In {\em Proceedings of the IEEE conference on computer vision and
  pattern recognition\/} (2017), pp.~652--660.

\bibitem{qi2017pointnet++}
{\sc Qi, C.~R., Yi, L., Su, H., and Guibas, L.~J.}
\newblock Pointnet++: Deep hierarchical feature learning on point sets in a
  metric space.
\newblock {\em Advances in neural information processing systems 30\/} (2017).

\bibitem{qian2022pointnext}
{\sc Qian, G., Li, Y., Peng, H., Mai, J., Hammoud, H., Elhoseiny, M., and
  Ghanem, B.}
\newblock Pointnext: Revisiting pointnet++ with improved training and scaling
  strategies.
\newblock {\em Advances in neural information processing systems 35\/} (2022),
  23192--23204.

\bibitem{reininghaus2015stable}
{\sc Reininghaus, J., Huber, S., Bauer, U., and Kwitt, R.}
\newblock A stable multi-scale kernel for topological machine learning.
\newblock In {\em Proceedings of the IEEE conference on computer vision and
  pattern recognition\/} (2015), pp.~4741--4748.

\bibitem{ren2022benchmarking}
{\sc Ren, J., Pan, L., and Liu, Z.}
\newblock Benchmarking and analyzing point cloud classification under
  corruptions.
\newblock In {\em International Conference on Machine Learning\/} (2022), PMLR,
  pp.~18559--18575.

\bibitem{sarker2024comprehensive}
{\sc Sarker, S., Sarker, P., Stone, G., Gorman, R., Tavakkoli, A., Bebis, G.,
  and Sattarvand, J.}
\newblock A comprehensive overview of deep learning techniques for 3d point
  cloud classification and semantic segmentation.
\newblock {\em Machine Vision and Applications 35}, 4 (2024), 67.

\bibitem{sedaghat2016orientation}
{\sc Sedaghat, N., Zolfaghari, M., Amiri, E., and Brox, T.}
\newblock Orientation-boosted voxel nets for 3d object recognition.
\newblock In {\em British Machine Vision Conference 2017, {BMVC} 2017, London,
  UK, September 4-7, 2017\/} (2017), {BMVA} Press.

\bibitem{sfikas2018ensemble}
{\sc Sfikas, K., Pratikakis, I., and Theoharis, T.}
\newblock Ensemble of panorama-based convolutional neural networks for 3d model
  classification and retrieval.
\newblock {\em Computers \& Graphics 71\/} (2018), 208--218.

\bibitem{singh2023topological}
{\sc Singh, Y., Farrelly, C.~M., Hathaway, Q.~A., Leiner, T., Jagtap, J.,
  Carlsson, G.~E., and Erickson, B.~J.}
\newblock Topological data analysis in medical imaging: current state of the
  art.
\newblock {\em Insights into Imaging 14}, 1 (2023), 58.

\bibitem{skaf2022topological}
{\sc Skaf, Y., and Laubenbacher, R.}
\newblock Topological data analysis in biomedicine: A review.
\newblock {\em Journal of Biomedical Informatics 130\/} (2022), 104082.

\bibitem{su2015multi}
{\sc Su, H., Maji, S., Kalogerakis, E., and Learned-Miller, E.}
\newblock Multi-view convolutional neural networks for 3d shape recognition.
\newblock In {\em Proceedings of the IEEE international conference on computer
  vision\/} (2015), pp.~945--953.

\bibitem{sun2022benchmarking}
{\sc Sun, J., Zhang, Q., Kailkhura, B., Yu, Z., Xiao, C., and Mao, Z.~M.}
\newblock Benchmarking robustness of 3d point cloud recognition against common
  corruptions.
\newblock {\em arXiv preprint arXiv:2201.12296\/} (2022).

\bibitem{tauzin2021giotto}
{\sc Tauzin, G., Lupo, U., Tunstall, L., P{\'e}rez, J.~B., Caorsi, M.,
  Medina-Mardones, A.~M., Dassatti, A., and Hess, K.}
\newblock giotto-tda:: A topological data analysis toolkit for machine learning
  and data exploration.
\newblock {\em Journal of Machine Learning Research 22}, 39 (2021), 1--6.

\bibitem{turner2014persistent}
{\sc Turner, K., Mukherjee, S., and Boyer, D.~M.}
\newblock Persistent homology transform for modeling shapes and surfaces.
\newblock {\em Information and Inference: A Journal of the IMA 3}, 4 (2014),
  310--344.

\bibitem{uy2019revisiting}
{\sc Uy, M.~A., Pham, Q.-H., Hua, B.-S., Nguyen, T., and Yeung, S.-K.}
\newblock Revisiting point cloud classification: A new benchmark dataset and
  classification model on real-world data.
\newblock In {\em Proceedings of the IEEE/CVF international conference on
  computer vision\/} (2019), pp.~1588--1597.

\bibitem{wagner2011efficient}
{\sc Wagner, H., Chen, C., and Vu{\c{c}}ini, E.}
\newblock Efficient computation of persistent homology for cubical data.
\newblock In {\em Topological methods in data analysis and visualization II:
  theory, algorithms, and applications}. Springer, 2011, pp.~91--106.

\bibitem{wang2023octformer}
{\sc Wang, P.-S.}
\newblock Octformer: Octree-based transformers for 3d point clouds.
\newblock {\em ACM Transactions on Graphics (TOG) 42}, 4 (2023), 1--11.

\bibitem{wang2019dynamic}
{\sc Wang, Y., Sun, Y., Liu, Z., Sarma, S.~E., Bronstein, M.~M., and Solomon,
  J.~M.}
\newblock Dynamic graph cnn for learning on point clouds.
\newblock {\em ACM Transactions on Graphics (tog) 38}, 5 (2019), 1--12.

\bibitem{wasserman2018topological}
{\sc Wasserman, L.}
\newblock Topological data analysis.
\newblock {\em Annual review of statistics and its application 5}, 2018 (2018),
  501--532.

\bibitem{wu2016learning}
{\sc Wu, J., Zhang, C., Xue, T., Freeman, B., and Tenenbaum, J.}
\newblock Learning a probabilistic latent space of object shapes via 3d
  generative-adversarial modeling.
\newblock {\em Advances in neural information processing systems 29\/} (2016).

\bibitem{wu2019point}
{\sc Wu, P., Chen, C., Yi, J., and Metaxas, D.}
\newblock Point cloud processing via recurrent set encoding.
\newblock In {\em Proceedings of the AAAI conference on artificial
  intelligence\/} (2019), vol.~33, pp.~5441--5449.

\bibitem{wu2023omniobject3d}
{\sc Wu, T., Zhang, J., Fu, X., Wang, Y., Ren, J., Pan, L., Wu, W., Yang, L.,
  Wang, J., Qian, C., et~al.}
\newblock Omniobject3d: Large-vocabulary 3d object dataset for realistic
  perception, reconstruction and generation.
\newblock In {\em Proceedings of the IEEE/CVF Conference on Computer Vision and
  Pattern Recognition\/} (2023), pp.~803--814.

\bibitem{wu20153d}
{\sc Wu, Z., Song, S., Khosla, A., Yu, F., Zhang, L., Tang, X., and Xiao, J.}
\newblock 3d shapenets: A deep representation for volumetric shapes.
\newblock In {\em Proceedings of the IEEE conference on computer vision and
  pattern recognition\/} (2015), pp.~1912--1920.

\bibitem{xiang2021walk}
{\sc Xiang, T., Zhang, C., Song, Y., Yu, J., and Cai, W.}
\newblock Walk in the cloud: Learning curves for point clouds shape analysis.
\newblock In {\em Proceedings of the IEEE/CVF international conference on
  computer vision\/} (2021), pp.~915--924.

\bibitem{xu2021learning}
{\sc Xu, M., Zhang, J., Zhou, Z., Xu, M., Qi, X., and Qiao, Y.}
\newblock Learning geometry-disentangled representation for complementary
  understanding of 3d object point cloud.
\newblock In {\em Proceedings of the AAAI conference on artificial
  intelligence\/} (2021), vol.~35, pp.~3056--3064.

\bibitem{xu2020grid}
{\sc Xu, Q., Sun, X., Wu, C.-Y., Wang, P., and Neumann, U.}
\newblock Grid-gcn for fast and scalable point cloud learning.
\newblock In {\em Proceedings of the IEEE/CVF conference on computer vision and
  pattern recognition\/} (2020), pp.~5661--5670.

\bibitem{xu2018spidercnn}
{\sc Xu, Y., Fan, T., Xu, M., Zeng, L., and Qiao, Y.}
\newblock Spidercnn: Deep learning on point sets with parameterized
  convolutional filters.
\newblock In {\em Proceedings of the European conference on computer vision
  (ECCV)\/} (2018), pp.~87--102.

\bibitem{yang2021reinventing}
{\sc Yang, J., Huang, X., He, Y., Xu, J., Yang, C., Xu, G., and Ni, B.}
\newblock Reinventing 2d convolutions for 3d images.
\newblock {\em IEEE Journal of Biomedical and Health Informatics 25}, 8 (2021),
  3009--3018.

\bibitem{yang2023medmnist}
{\sc Yang, J., Shi, R., Wei, D., Liu, Z., Zhao, L., Ke, B., Pfister, H., and
  Ni, B.}
\newblock Medmnist v2-a large-scale lightweight benchmark for 2d and 3d
  biomedical image classification.
\newblock {\em Scientific Data 10}, 1 (2023), 41.

\bibitem{yang2020intra}
{\sc Yang, X., Xia, D., Kin, T., and Igarashi, T.}
\newblock Intra: 3d intracranial aneurysm dataset for deep learning.
\newblock In {\em Proceedings of the IEEE/CVF conference on computer vision and
  pattern recognition\/} (2020), pp.~2656--2666.

\bibitem{you2018pvnet}
{\sc You, H., Feng, Y., Ji, R., and Gao, Y.}
\newblock Pvnet: A joint convolutional network of point cloud and multi-view
  for 3d shape recognition.
\newblock In {\em Proceedings of the 26th ACM international conference on
  Multimedia\/} (2018), pp.~1310--1318.

\bibitem{yu2018multi}
{\sc Yu, T., Meng, J., and Yuan, J.}
\newblock Multi-view harmonized bilinear network for 3d object recognition.
\newblock In {\em Proceedings of the IEEE conference on computer vision and
  pattern recognition\/} (2018), pp.~186--194.

\bibitem{yu2022point}
{\sc Yu, X., Tang, L., Rao, Y., Huang, T., Zhou, J., and Lu, J.}
\newblock Point-bert: Pre-training 3d point cloud transformers with masked
  point modeling.
\newblock In {\em Proceedings of the IEEE/CVF conference on computer vision and
  pattern recognition\/} (2022), pp.~19313--19322.

\bibitem{zhao2022rotation}
{\sc Zhao, C., Yang, J., Xiong, X., Zhu, A., Cao, Z., and Li, X.}
\newblock Rotation invariant point cloud analysis: Where local geometry meets
  global topology.
\newblock {\em Pattern Recognition 127\/} (2022), 108626.

\bibitem{zhao2021point}
{\sc Zhao, H., Jiang, L., Jia, J., Torr, P.~H., and Koltun, V.}
\newblock Point transformer.
\newblock In {\em Proceedings of the IEEE/CVF international conference on
  computer vision\/} (2021), pp.~16259--16268.

\bibitem{zomorodian2004computing}
{\sc Zomorodian, A., and Carlsson, G.}
\newblock Computing persistent homology.
\newblock {\em Discrete \& Computational Geometry 33}, 2 (2004), 249--274.

\end{thebibliography}

\newpage
\section*{Appendix}

\subsection{\taco~ Pipeline and Illustration of Cubical Persistence}

We present a diagram of the \taco~ pipeline in Fig.~\ref{fig:pipeline}.

\begin{figure}[h]
    \centering
    \includegraphics[width=1.0\linewidth]{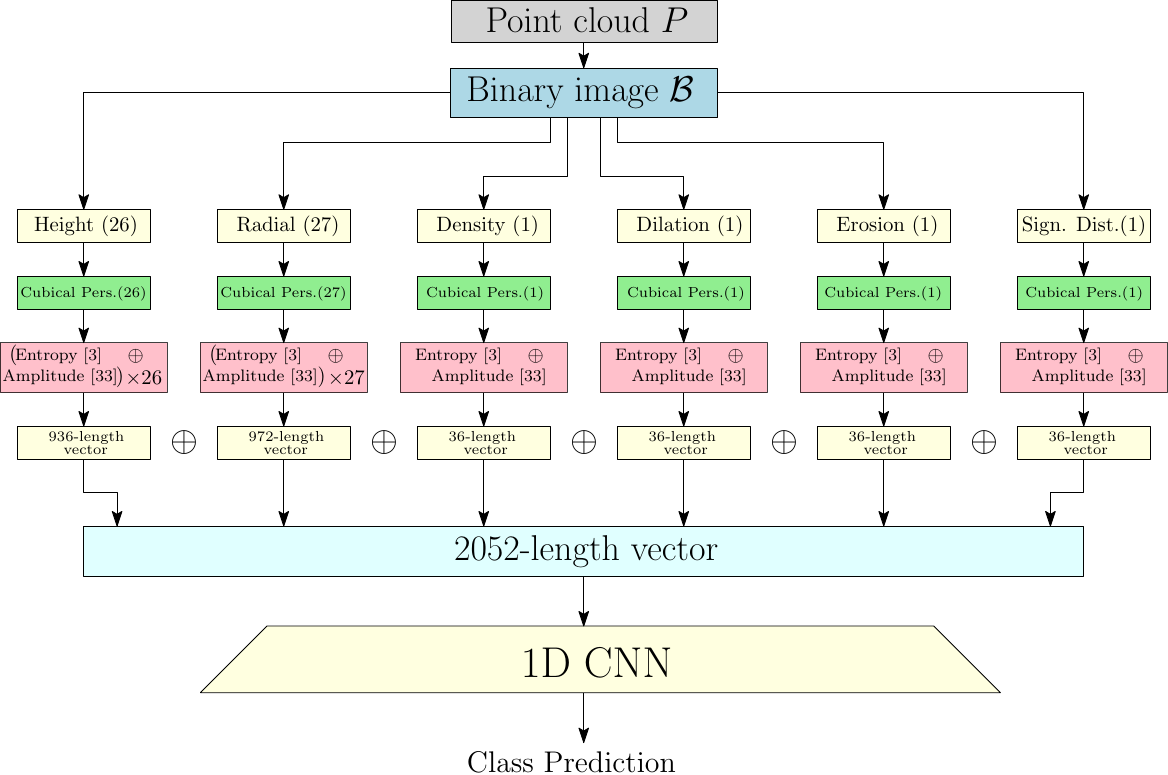}
    \caption{An illustration of the novel pipeline of \taco~. The numbers in the parentheses denote the number of variants. For instance, {`Height (26)'} implies that $26$ variants of the height filtration have been used. The integers inside the square braces denote vector length. For example, Entropy[3] implies that applying entropy to cubical persistence yields a vector of length 3.
    Further, $\oplus$ denotes vector concatenation. }
    \label{fig:pipeline}
\end{figure}

\noindent
\textbf{An illustrative example. }
We give an example of cubical persistence in Fig.~\ref{fig:example}(a-d). 

A 2D grayscale image with pixels having their grayscale values in $\{0,50,100\}$ is shown in (a). During filtration, $0$-pixels are considered first, then $50$-pixels, and finally $100$-pixels. In (b), $K_0$ is shown; the $0$-pixels are added, resulting in three connected components, each comprising just one pixel. (c) $K_1$:  The $50$-pixels are added. Consequently, there is just one connected component that looks like the digit $6$. In the previous step, there were three, but now just one. So, three connected components took birth at $0$, and two of them died at $50$. A hole takes birth inside $K_1$. (d) $K_2$: The $100$-pixels are added. The hole obtained in the previous step dies in this step. 

In the homology dimension $0$, there are two birth-death pairs $(0,50), (0,50)$ since two connected components died, and in dimension $1$, there is exactly one $(50,100)$ since the hole formed at $50$ and died at $100$. Hence, the persistence (a multiset) contains three pairs. This is expressed pictorially in Fig.~\ref{fig:pd}. In the end, there is just one connected component that never dies. 

\begin{figure}[h]
    \centering
    \includegraphics[width=1\linewidth]{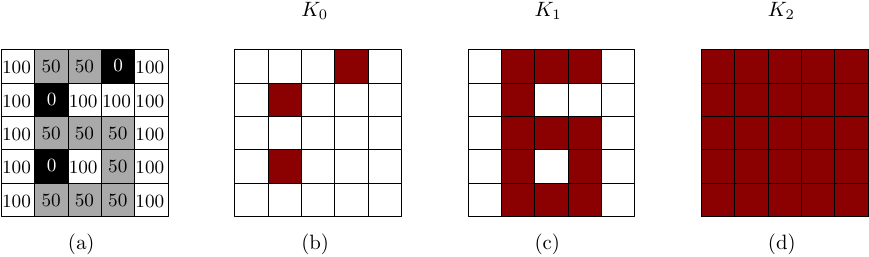}
    \caption{Illustrating filtration for cubical persistence (here shown in 2D). Note that $K_0 \subseteq K_1 \subseteq K_2$. }    
    \label{fig:example}
\end{figure}

\begin{figure}[h]
    \centering
    \includegraphics[width=0.5\linewidth]{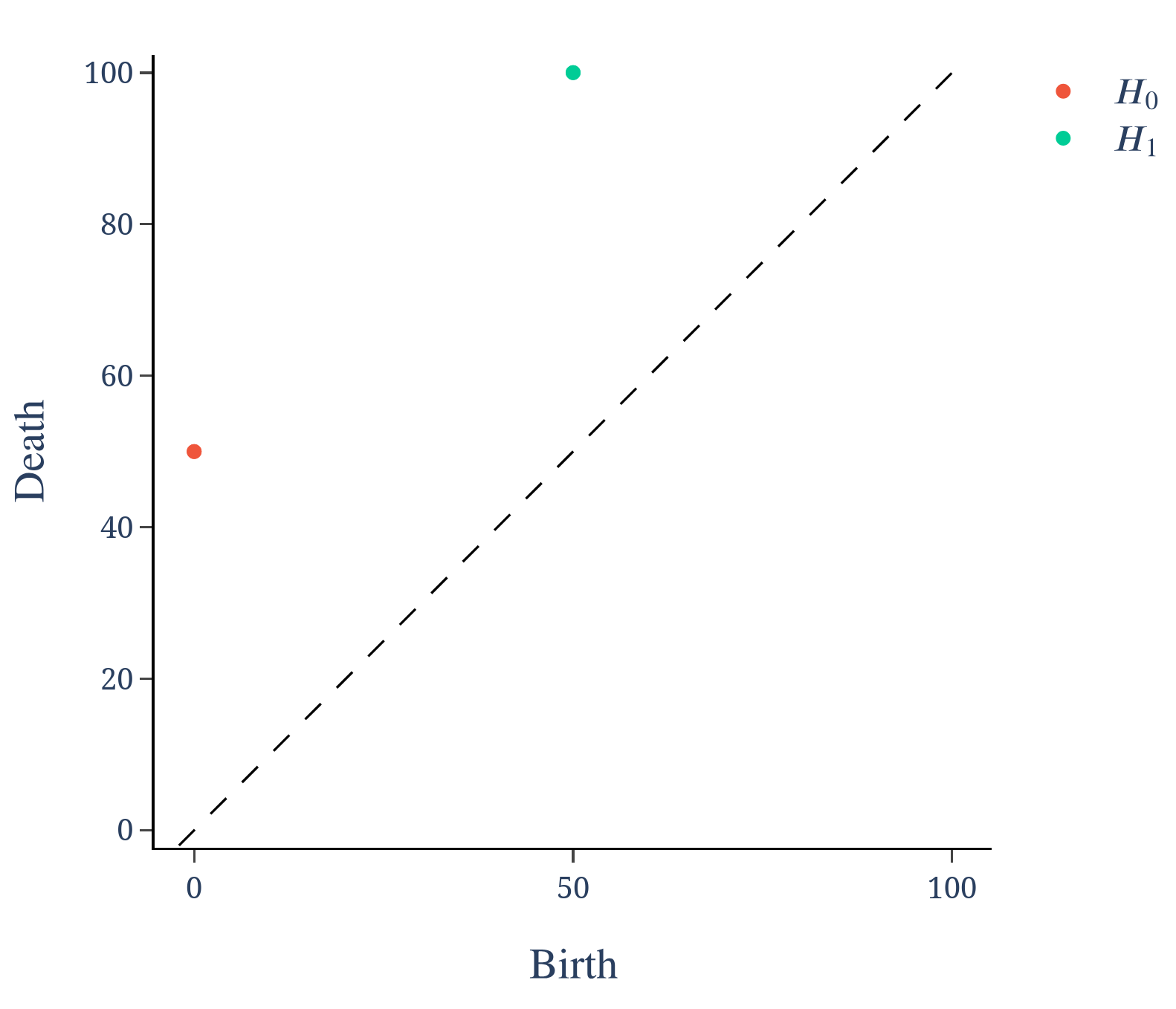}
    \caption{The persistence diagram corresponding to the filtration shown above. There are two overlapping red dots for the two birth-death pairs $(0,50), (0,50)$. The green dot corresponds to the pair $(50,100)$.}    
    \label{fig:pd}
\end{figure}

\noindent
\textbf{Rationale for the Amplitude Kinds. }

In what follows, we provide an expansion on the amplitude discussion presented in Sec.~\ref{sec:tda}, to provide insights on their use.

\textit{$p$-Wasserstein}~\cite{tauzin2021giotto}. The \textit{half-lifetime} of a pair $(b_i,d_i)\in X$ is defined as $\frac{d_i-b_i}{2}$. The Wasserstein amplitude of order $p$, denoted by $W(X,p)$, is defined as the $L_p$ norm of the vector of half-lifetimes of the birth-death pairs in $X$. Hence, $W(X,p) = ( \sum_{i} \left( \frac{d_i - b_i}{2}\right)^p )^{1/p}.$ For~\taco~, we have used $p=1,2$. We obtain $6$ real numbers for this metric, since there are $3$ homology groups and $2$ values of $p$. This metric aids in measuring the $L_p$ norm of half‑lifetimes, providing a stable (to diagram perturbations) and tunable sensitivity to feature persistence.
Using $p=1$ emphasizes the aggregate contribution of many moderate features, while 
$p=2$ weights longer lifetimes, more strongly complementary views that improve discrimination.

\textit{Bottleneck}~\cite{tauzin2021giotto}. The Bottleneck amplitude is denoted by  $B(X) = W(X,\infty)$. We obtain $3$ real numbers for this metric due to the three homology groups. This metric captures the single most persistent topological structure, which often aligns with the dominant, category‑defining shape cue.
Its robustness to small noise and invariance to minor diagram perturbations make it a strong separator when a few long‑lived features matter most.

\textit{Betti curve.~\cite{tauzin2021giotto}} The Betti curve of $X$ is the function $B_C:\mathbf{R} \rightarrow \mathbf{N}$, such that $B_C(s)$ gives the number of birth-pairs in $X$ that contains $s$ when every pair $(b_i,d_i)$ in $X$ is treated as an interval. Two amplitudes are obtained using the $L_1$ and $L_2$ norms. We obtain $6$ real numbers for this metric, since there are three homology groups and two norms. This metric helps summarize ``how many features are alive" across filtration values, yielding an interpretable 1D profile of topology over scale.
The norms over this curve provide compact, stable vectors that are efficient to learn with a 1D CNN while retaining multi‑scale counts.

\textit{Landscape}~\cite{bubenik2017persistence, bubenik2015statistical}. For a birth-death pair  $(b_i,d_i) \in X$, let $f_{(b_i,d_i)} : \mathbf{R} \rightarrow [0,\infty]$, be a piecewise linear function given in the following equation. 

\begin{equation*}
 f_{(b_i,d_i)} (x) =  \begin{cases}
0 & \text{if } x \notin (b_i,d_i) \\
x-b_i & \text{if } x \in (b_i,\frac{b_i+d_i}{2}] \\
-x+d_i & \text{if } x \in  (\frac{b_i+d_i}{2}, d_i)
\end{cases}
\end{equation*}
%

The \textit{persistence landscape} of $X$ is the sequence of functions $\lambda_k : \mathbf{R} \rightarrow [0, \infty], k =1, 2, 3, \ldots$ where $\lambda_k (x)$ is the $k$-th largest value of $\{f_{(b_i,d_i)}(x)\}_i$. Further, $\lambda_k(x)$ is set to $0$ if the $k$-th largest value does not exist. 
The parameter $k$ is called the \textit{layer}. For~\taco~, we have used $k=1,2$. Four amplitudes are obtained using $L_1$ and $L_2$ norms for both the values of $k$. We get $12$ real numbers for this metric, since there are three homology groups, two norms, and two distinct values of $k$.
This metric encodes order‑statistics of feature prominence via layers $\lambda_k$, preserving more geometric detail than simple counts yet remaining Hilbert‑space friendly for averaging and norms.
Using $k=1,2$ captures dominant and secondary structures, offering a stable, rich functional summary that boosts class separability.

\textit{Heat kernel}~\cite{reininghaus2015stable}. Gaussians of standard deviation $\sigma$ are placed over every point in $X$ and a negative Gaussian of $\sigma$ on the mirror point across the diagonal line in the persistence diagram. Thus, a real-valued function is obtained on $\mathbf{R}^2$. For \taco~, we have used $\sigma=0.15$. We get $6$ real numbers for this metric, since there are three homology groups and two norms, $L_1,L_2$. This metric places (positive) Gaussians on diagram points and (negative) mirrors across the diagonal, yielding a smooth, multi‑scale similarity that is robust to small birth/death shifts.
This continuous embedding captures spatial arrangement in the diagram and works well with standard norms; our 
$\sigma=0.15$ balances noise‑tolerance and sensitivity.

\subsection{Theoretical Analysis}

\begin{theorem}
Let $P$ be an $n$-element point cloud that needs to be classified by~\taco~. Then, the time taken for classification is $\mathcal{O}(n + v^3 + v/\rho^3 )$, where $v$ is the number of voxels in $\mathcal{B}$ and $\rho$ is the voxel-size used.
\end{theorem}

\begin{proof} 
Initializing all voxels in \( \mathcal{B} \) as inactive requires \( \mathcal{O}(v) \) time. For each point in $P$, locating the corresponding voxel in \( \mathcal{B} \) takes \( \mathcal{O}(1) \) time. Since \( |P| = n \), the total time to prepare \( \mathcal{B} \) is \( \mathcal{O}(v + n) \).

From \( \mathcal{B} \), we generate 57 grayscale images using six filtration types: height, radial, density, dilation, erosion, and signed distance. For the 26 height and 27 radial filtrations, each voxel requires a constant-time distance computation, resulting in \( \mathcal{O}(v) \) time per filtration. Thus, the total time for generating these 53 grayscale images is \( \mathcal{O}(57 \cdot v) = \mathcal{O}(v) \), including initialization. For the density filtration, each voxel must inspect its neighborhood within a ball of radius \( r \), which contains \( \mathcal{O}(r^3/\rho^3) \) voxels. In \taco~, we set \( r = 1 \), yielding a per-voxel cost of \( \mathcal{O}(1/\rho^3) \), and a total cost of \( \mathcal{O}(v/\rho^3) \). The remaining three filtrations, dilation, erosion, and signed distance, can be computed in \( \mathcal{O}(v) \) time each using efficient distance transform algorithms~\cite{fabbri20082d}. Therefore, the total time for generating all 57 grayscale images is \( \mathcal{O}(v(1 + 1/\rho^3)) \).

Cubical persistence for a single grayscale image can be computed in \( \mathcal{O}(v^3) \) time using standard matrix reduction techniques~\cite{wagner2011efficient}. For 57 images, the total cost is \( \mathcal{O}(57 \cdot v^3) = \mathcal{O}(v^3) \). Each voxel generates up to 27 cells (1 cube, 6 faces, 12 edges, 8 vertices), the building blocks of a cubical complex. Each cell can belong to at most one persistence pair.
Hence, the worst-case number of birth–death pairs is at most $27v=\mathcal{O}(v)$.

Feature extraction from each cubical persistence involves computing persistent entropy and amplitude metrics. Persistent entropy requires \( \mathcal{O}(v) \) time per image. Wasserstein and Bottleneck amplitudes also take \( \mathcal{O}(v) \) time. The Betti curve kernel, evaluated over 100 filtration values, requires \( \mathcal{O}(v) \) time per image. The persistence landscape kernel, which involves sorting at each of 100 sampled values, incurs \( \mathcal{O}(v \log v) \) time. The heat kernel, evaluated over 20 filtration values, takes \( \mathcal{O}(v) \) time. Thus, the total time to generate the topological features for one grayscale image is \( \mathcal{O}(v \log v) \). For the $57$ images, time taken is $\mathcal{O}(57\cdot v \log v) = \mathcal{O}(v\log v)$.

Since the 1D CNN model is fixed during inference, classification of the topological vector takes constant time, i.e., \( \mathcal{O}(1) \).

Combining all components, the overall time complexity of the pipeline is:\\
$$\mathcal{O}(v + n) + \mathcal{O}\left(v\left(1 + \frac{1}{\rho^3}\right)\right) + \mathcal{O}(v^3) + \mathcal{O}(v \log v) + \mathcal{O}(1) = \mathcal{O}\left(n + v^3 + \frac{v}{\rho^3} \right).
$$
\end{proof}

\subsection{Further Experimental Details and Results}

\subsubsection{Datasets}
\label{sec:datasets}
We have used the following six datasets to test the performance of the proposed \taco~ framework.
\vspace{-5pt}
\begin{itemize}[leftmargin=*]
    \item ModelNet40~\cite{wu20153d}: It is one of the most popular benchmark datasets. The dataset comprises 40 classes, each consisting of CAD models of everyday objects. We used the official split, which consisted of 9,843 shapes for training and 2,468 for testing. For every shape, a random uniform sample of 2048 3D points was extracted from these CAD objects for the classification task.  
    \item ModelNet10: A smaller, 10-class version of ModelNet40~\cite{wu20153d} with $3991$ train and $908$ test samples. Similar to ModelNet40, a $2048$-element point cloud for each object was used in our experiments.
    \item OmniObject3D: A real-world point cloud object dataset, which is notoriously difficult to classify~\cite{wu2023omniobject3d}. With a large number of categories, it poses an extreme class-imbalance and inter-class similarity challenge, making accurate classification significantly harder compared to smaller-scale benchmarks. Unlike ModelNet10/40, there is no official train/test split available for this dataset. Therefore, we used an $80/20$ split, without instance leakage. 
    \item ScanObjectNN: The challenging real-world OBJ\_BG variant is derived from scanned indoor scenes, comprising ($2309$ train and $581$ test) partial and noisy point clouds with backgrounds across $15$ object classes, often with multiple objects co-existing in cluttered environments~\cite{uy2019revisiting}.For this, the voxel size $\rho=0.03$. The vector length is set to 1440, the visual reasoning for which is presented in Fig. \ref{fig:Sonn_lvsacc}.
    \item VesselMNIST3D: In \cite{yang2020intra}, the authors have introduced an open-access 3D intracranial aneurysm dataset with $103$ 3D meshes from brain Magnetic Resonance Angiography (MRA). This dataset has two classes: $1,694$ healthy vessel segments (V.) and $215$ aneurysm segments (A.). The dataset has the training, validation, and test set ratio of $7:1:2$~\cite{yang2023medmnist}. 

    \item AdrenalMNIST3D: It is a CT scan dataset with two classes, consisting of shape masks from $1,584$ left and right adrenal glands (i.e., $792$ patients)~\cite{yang2023medmnist}. The binarized images for the two medical datasets are provided through the \texttt{medmnist} Python package and have a resolution of $28 \times 28 \times 28$. Therefore, our starting point is 3D binary images instead of point clouds for AdrenalMNIST3D and VesselMNIST3D.
\end{itemize}

\begin{figure}[ht!]
  \centering
  \begin{tabular}{cc}
      \hspace{-0.1in}\includegraphics[scale=0.5]{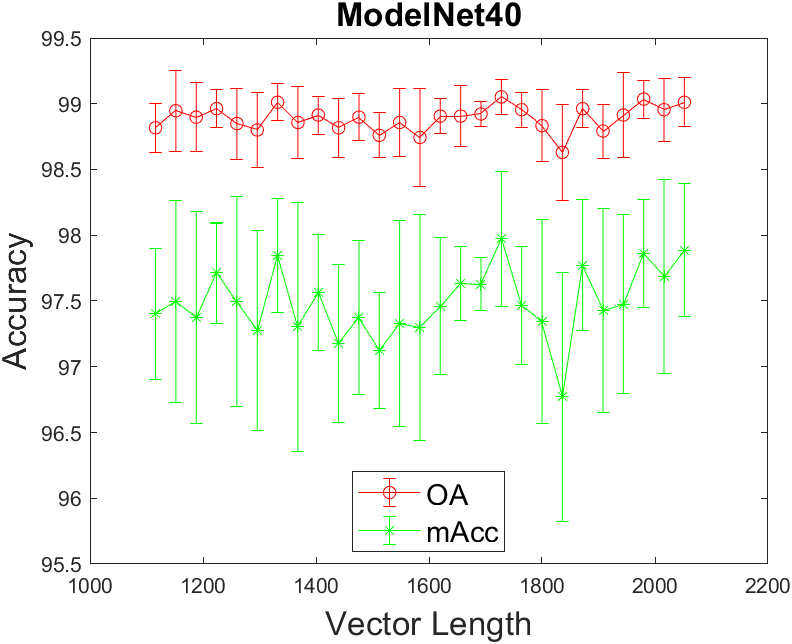} & \hspace{-0.2in}\includegraphics[scale=0.35]{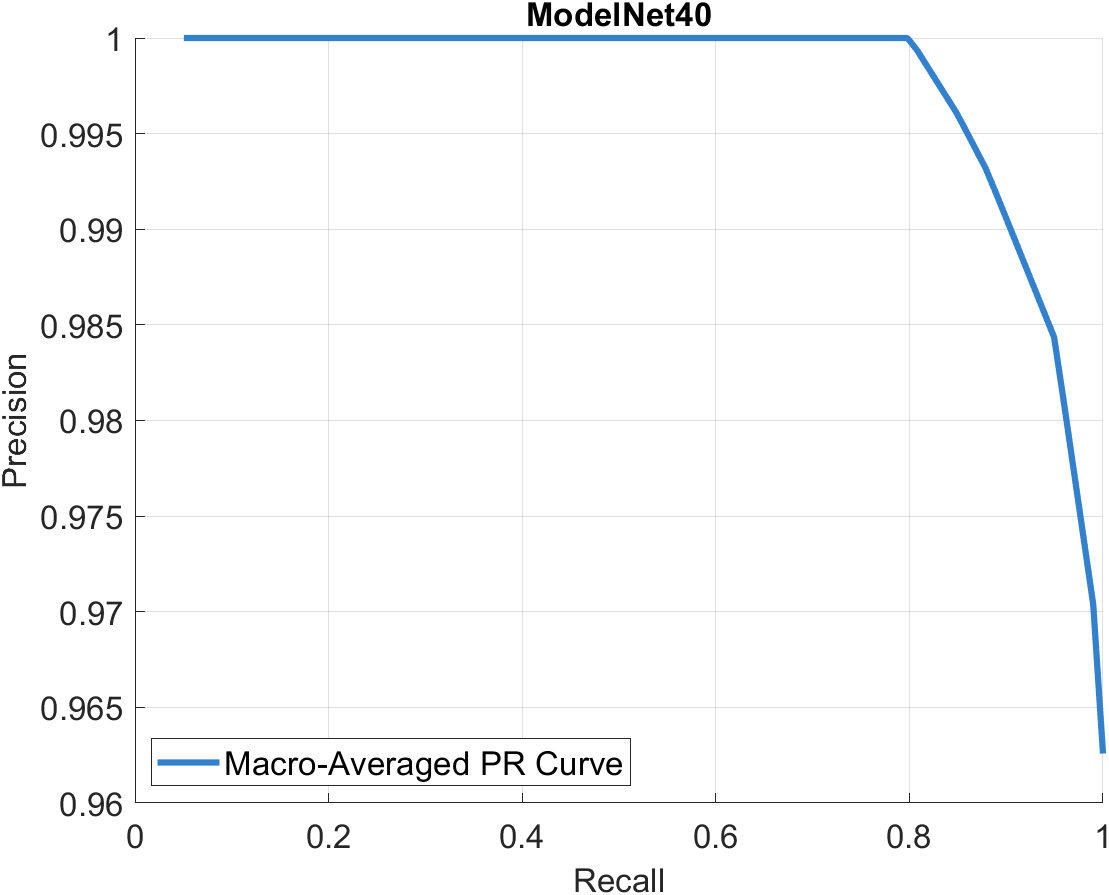} \\
      (a) & (b)
  \end{tabular}
  \caption{ModelNet40: (a) Change in OA/mAcc w.r.t. feature vector length (i.e., different number of radial filtrations) and (b) Precision-recall curve for the best OA model.}
  \label{fig:MN40_length_pr}
\end{figure}

\begin{figure}[ht!] 
  \centering
  \includegraphics[width=1.0\textwidth]{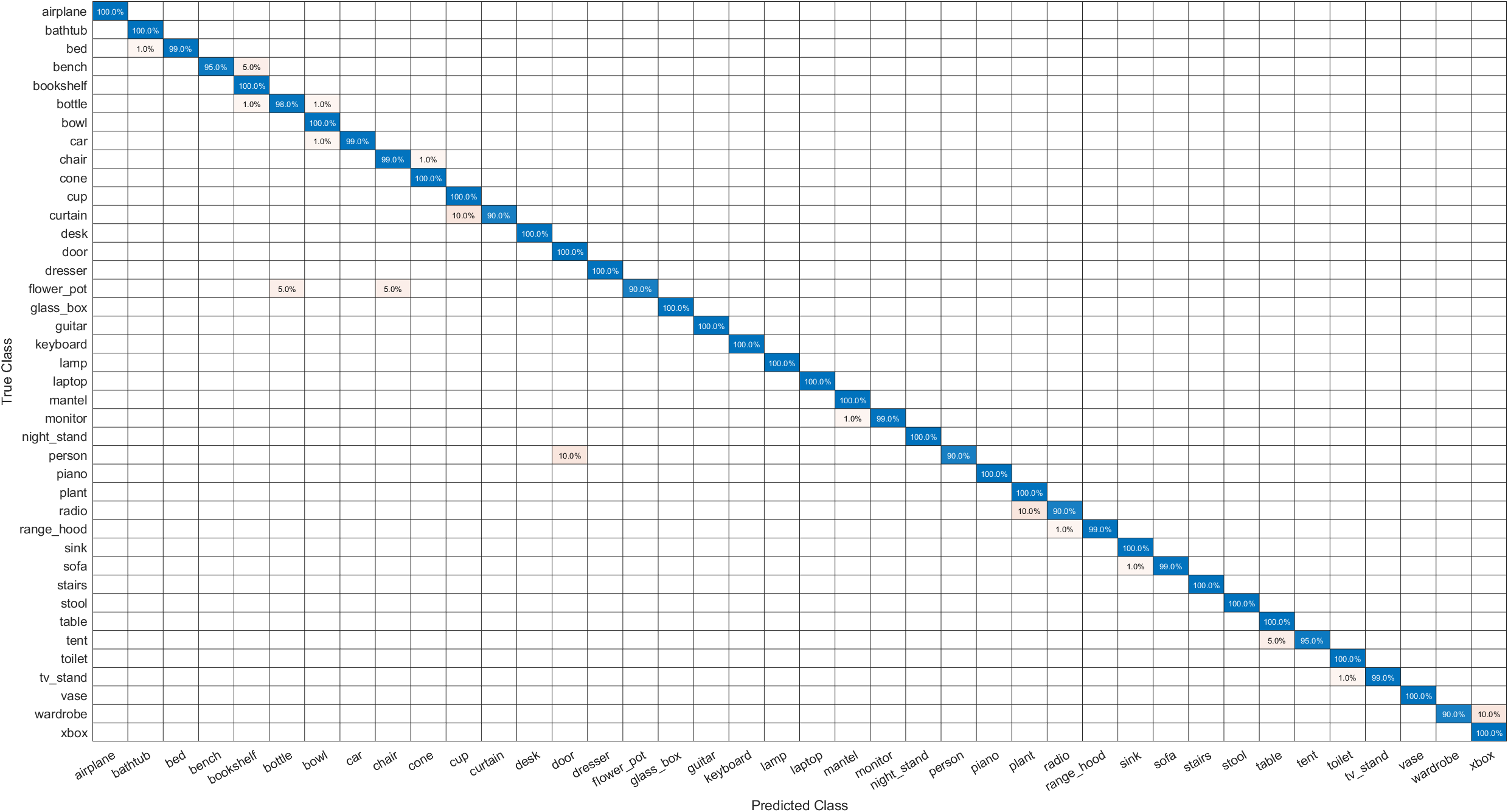} 
  \caption{Confusion matrix for the best OA model.}
  \label{fig:MN40_cm}
\end{figure}

\begin{table}[ht]
\centering
\begin{tabular}{ccc}
\toprule
\multicolumn{3}{c}{\textbf{ModelNet40}} \\
\midrule
{Method} & {OA } & {mAcc } \\
\midrule
3DShapeNets~\cite{wu20153d} & 84.7 & 77.3 \\
PointNet~\cite{qi2017pointnet} & 89.2 & 86.2 \\
MVCNN~\cite{su2015multi} & 90.1 & - \\
Ma et al.~\cite{ma2018learning} & 91.05 & - \\
KD-Network~\cite{klokov2017escape} & 91.8 & 88.5 \\
PointNet++~\cite{qi2017pointnet++} & 91.9 & - \\
PointCNN~\cite{li2018pointcnn} & 92.5 & 88.1 \\
OctFormer~\cite{wang2023octformer} & 92.7 & -\\
GVCNN~\cite{feng2018gvcnn} & 93.1 & - \\
PointNeXt~\cite{qian2022pointnext} & 93.2 & 90.8 \\
PointMamba~\cite{liang2024pointmamba} & 93.6 & -\\
Point-Transformer~\cite{zhao2021point} & 93.7 & 90.6\\
Point-Bert~\cite{yu2022point} & 93.8 & -\\
PointMLP~\cite{ma2022rethinking} & 94.5 & 91.4 \\
MHBN~\cite{yu2018multi} & 94.7 & 93.1 \\
PointGPT~\cite{chen2023pointgpt} & 94.9 & -\\
Pointview-GCN~\cite{mohammadi2021pointview} & 95.4 & - \\
VRN Ensemble~\cite{brock2016generative} & 95.54 & - \\
HGNN~\cite{feng2019hypergraph} & 96.6 & - \\
RotationNet~\cite{kanezaki2019rotationnet} & 97.37 & 96.29 \\
\midrule
\taco~ (Ours) & \textbf{99.05} & \textbf{97.97} \\
\bottomrule
\end{tabular}
\caption{Classification accuracy (\%) results on ModelNet40 dataset.}
\label{tab:modelnet40_results}
\end{table}

\begin{table}[ht]
\centering
\begin{tabular}{ccc}
\toprule
\multicolumn{3}{c}{\textbf{ModelNet10}} \\
\midrule
{Method} & {OA } & {mAcc } \\
\midrule
3DShapeNets~\cite{wu20153d} & 83.54 & - \\
3D-GAN~\cite{wu2016learning} & 91 & - \\
VoxNet~\cite{maturana2015voxnet} & 92 & - \\
Primitive-GAN~\cite{khan2019unsupervised} & 92.2 & - \\
ORION~\cite{sedaghat2016orientation} & 93.9 & - \\
KD-Network~\cite{klokov2017escape} & 94 & 93.5 \\
MHBN~\cite{yu2018multi} & 95 & 95 \\
3DmFV-Net~\cite{ben20183dmfv} & 95.2 & - \\
Point2Sequence~\cite{liu2019point2sequence} & 95.3 & 95.1 \\
A-CNN~\cite{komarichev2019cnn} & 95.5 & 95.3 \\
RCNet-E~\cite{wu2019point} & 95.6 & - \\
PANORAMA-ENN~\cite{sfikas2018ensemble} & 96.85 & - \\
VRN Ensemble~\cite{brock2016generative} & 97.14 & - \\
Grid-GCN~\cite{xu2020grid} & 97.5 & 97.4 \\
RotationNet~\cite{kanezaki2019rotationnet} & 98.9 & - \\
\midrule
\taco~ (Ours) & \textbf{99.52} & \textbf{99.52} \\
\bottomrule
\end{tabular}
\caption{Classification accuracy (\%) results on ModelNet10 dataset.}
\label{tab:modelnet10_results}
\end{table}

\begin{table}
\centering
\begin{tabular}{lc}
\toprule
Method & OA \\
\midrule
DGCNN~\cite{wang2019dynamic}     & 44.8 \\
PointNet~\cite{qi2017pointnet}    & 46.6 \\
PointNet++~\cite{qi2017pointnet++} & 40.7 \\
RSCNN~\cite{liu2019relation}     & 39.3 \\
SimpleView~\cite{goyal2021revisiting} & 47.6 \\
GDANet~\cite{xu2021learning}   & 49.7 \\
CurveNet~\cite{xiang2021walk} & 50.0 \\
PCT~\cite{guo2021pct}         & 45.9 \\
RPC~\cite{ren2022benchmarking}         & 47.2 \\
\midrule
\taco~ (Ours) & \textbf{58.9}\\
\bottomrule
\end{tabular}
\caption{Classification accuracy (\%) results on OmniObject3D dataset~\cite{wu2023omniobject3d}.}
\label{tab:omniobject3d}
\end{table}

\begin{figure}[ht!]
    \centering
    \includegraphics[width=0.65\linewidth]{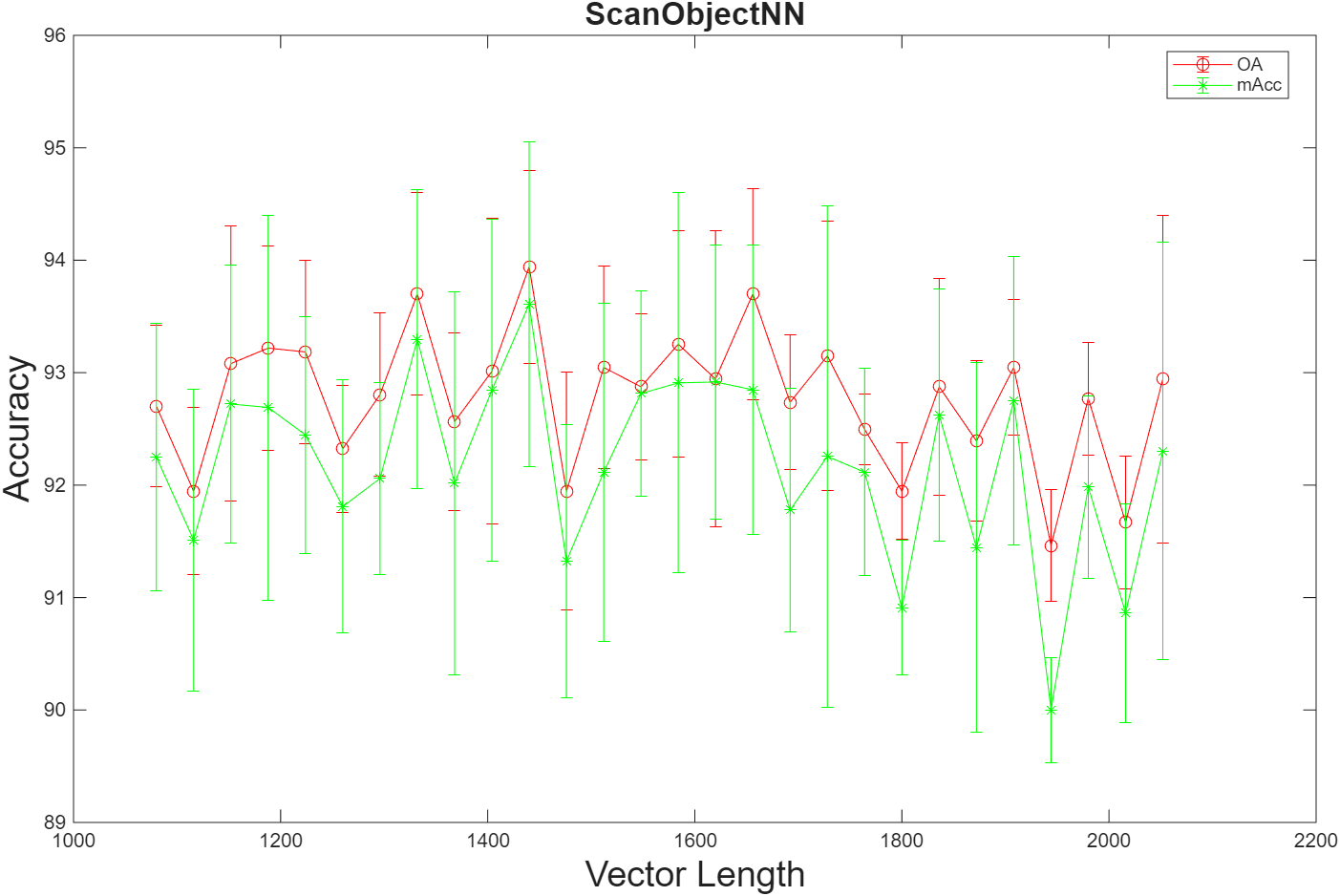}
    \caption{ScanObjectNN dataset: Change in OA/mAcc w.r.t. feature vector length (i.e., different number of radial filtrations).}
    \label{fig:Sonn_lvsacc}
\end{figure}

\begin{figure}[ht!]
    \centering
    \begin{tabular}{cc}
        \hspace{-0.1in}\includegraphics[scale=0.5]{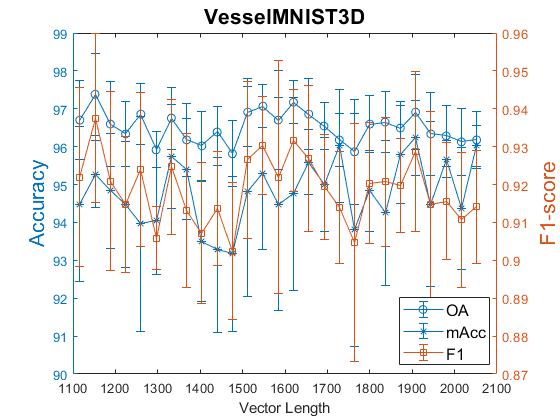} & 
        \hspace{-0.15in}\includegraphics[scale=0.5]{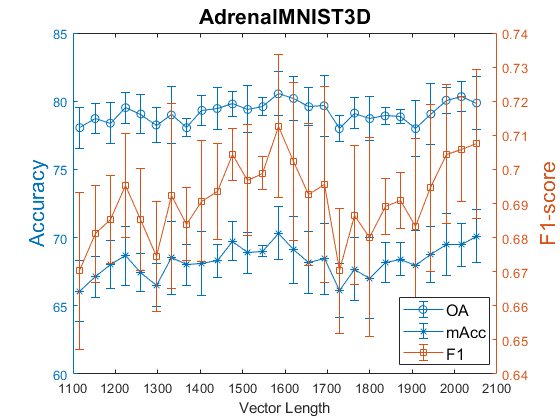}\\
    \end{tabular}
    \caption{Change in performance with various radial center points and consequent feature vector lengths. }
    \label{fig:vessel_adrenal}
\end{figure}

\begin{table}[ht]
\centering

\begin{minipage}[t]{0.48\textwidth}
\centering
\caption{VesselMNIST3D~\cite{yang2020intra}: The best mAcc (\%) and F1-scores from \cite{yang2020intra} are reported along with our average results at the bottom.}
\addtolength{\tabcolsep}{-0.4em}
\begin{tabular}{lcc}
\toprule
Network & mAcc & F1 \\
\midrule
PointNet++~\cite{qi2017pointnet++} & 93.52 & 0.90 \\
SpiderCNN~\cite{xu2018spidercnn} & 92.59 & 0.87 \\
SO-Net~\cite{li2018so} & 91.50 & 0.88 \\
PointCNN~\cite{li2018pointcnn} & 92.38 & 0.90 \\
DGCNN~\cite{wang2019dynamic} & 90.67 & 0.86 \\
PointNet~\cite{qi2017pointnet} & 81.62 & 0.69 \\
\midrule
\taco~ (Ours) & \textbf{95.28} & \textbf{0.94} \\
\bottomrule
\end{tabular}
\label{tab:vessel3d}
\end{minipage}
\hfill
\begin{minipage}[t]{0.48\textwidth}
\centering
\caption{Overall accuracies (OA) in \% for AdrenalMNIST3D (A3D) and VesselMNIST3D (V3D) across different methods as benchmarked in \cite{yang2023medmnist} compared with \taco~.}
\addtolength{\tabcolsep}{-0.4em}
\begin{tabular}{lcc}
\toprule
Methods & A3D & V3D \\
\midrule
ResNet-18 + 2.5D & 77.2 & 84.6 \\
ResNet-18 + 3D & 72.1 & 87.7 \\
ResNet-18 + ACS & 75.4 & 92.8 \\
ResNet-50 + 2.5D & 76.3 & 87.7 \\
ResNet-50 + 3D & 74.5 & 91.8 \\
ResNet-50 + ACS & 75.8 & 85.8 \\
auto-sklearn~\cite{feurer2015efficient} & 80.2 & 91.5 \\
AutoKeras~\cite{jin2019auto} & 70.5 & 89.4 \\
\midrule
\taco~ (Ours) & \textbf{80.54} & \textbf{97.38} \\
\bottomrule
\end{tabular}
\label{tab:adrenal}
\end{minipage}

\end{table}

\begin{table}[ht!]
    \centering
    \caption{Accuracies (\%) when trained on clean ModelNet40 and tested on perturbed ModelNet40 test set. 
    }

    \begin{tabular}{lllll}
    \toprule
    &\multicolumn{2}{c}{Low}&\multicolumn{2}{c}{High}\\
    \cmidrule{2-5}
    Perturbation & OA & mAcc & OA & mAcc\\ \midrule
    
    None (best model) & \multicolumn{2}{c}{--} & 99.15 & 98.37\\
    \midrule
   Downsampling & 99.11 & 98.34 & 85.21 & 83.09\\
  Uniform & 98.82 & 97.17 & 70.10 & 63.08\\
   Gaussian &  94.08 & 90.35 & 54.50 & 47.58\\

  Upsampling &  91.90 & 88.02 & 50.28 & 47.19\\

Rotation &  97.49 & 95.93 & 49.39 & 45.00\\
   Shear &  98.74 & 97.82 & 61.91 & 62.89\\
  FFD &  98.70 & 97.69 & 76.34 & 71.16\\

   RBF  & 99.07 & 98.02 & 83.47 & 77.07\\
   Inverse-RBF & 99.15 & 98.47 & 86.63 & 80.04\\
   Impulse & 52.88 & 41.80 & 24.19 & 18.82\\
     \bottomrule
    \end{tabular}
    \label{tab:modelnet-C}
\end{table}

\begin{table}[ht]
\caption{Comparison with XGBoost and Random Forest classifiers (dataset: ModelNet40)}
    \centering
    \begin{tabular}{cccc}
    \toprule
        Algorithm & OA (\%) & Training Time (mins.) & Test Throughput\\
        \midrule
        Random Forest & 78.35 & \textbf{0.55} & 18,985\\
        XGBoost & 81.35 & 4.35  & \textbf{61,700}\\
        \midrule
        \taco~ (Ours) & \textbf{99.05} & 2.50 & 16,454\\
        \bottomrule
    \end{tabular}
    \label{tab:xgboost_comp}
\end{table}

\begin{table}[ht]
\centering
\begin{tabular}{lc}
\toprule
Methods & Param. (M) \\
\midrule
PointNet~\cite{qi2017pointnet} & 3.5 \\
PointNet++~\cite{qi2017pointnet++} & 1.5 \\
MVTN~\cite{hamdi2021mvtn} & 3.5 \\
DGCNN~\cite{wang2019dynamic} & 1.8 \\
PointNeXt~\cite{qian2022pointnext} & 1.4 \\
PCT~\cite{guo2021pct} & 2.9 \\
Point-BERT~\cite{yu2022point} & 22.1 \\
PointGPT~\cite{chen2023pointgpt} & 29.2 \\
PointMamba~\cite{liang2024pointmamba} & 12.3 \\
\midrule
\taco~ (Ours) & \textbf{0.72} \\
\bottomrule
\end{tabular}
\caption{Comparison of network parameters (in millions) of different models for the ModelNet40 dataset.}
\label{tab:params_only}
\end{table}

\subsubsection{Comparison with non-deep learning algorithms}
We show that the proposed 1D CNN model thoroughly outperforms standard non-deep learning methods such as XGBoost~\cite{chen2016xgboost} and Random Forest classifiers. We chose the ModelNet40 dataset for this test while using the topological feature vector length of $1728$ as mentioned in the paper. For XGBoost and Random Forest classifiers, we used Python's \texttt{xgboost} and \texttt{scikit-learn} packages, respectively, with default options.

The comparison result is presented in Table \ref{tab:xgboost_comp}. Our proposed \taco~ achieves $17.7\%$ and $20.7\%$ higher accuracies than XGBoost and Random Forest, respectively. Notably, XGBoost required $74\%$ more training time than that of \taco~. On the other hand, the throughput of XGBoost was $2.75$x faster than  \taco~. Overall, these results empirically demonstrate the superiority of the 1D CNN within the proposed \taco~ framework in learning meaningful representations for object classification from the input topological vectors.

\begin{table}
\centering
\begin{tabular}{lc}
\toprule
Method & Retrieval mAP\\
\midrule
3D ShapeNets~\cite{wu20153d} & 49.2 \\
Densepoint~\cite{liu2019densepoint} & 88.5 \\
PVNet~\cite{you2018pvnet} & 89.5 \\
MVCNN~\cite{su2015multi} & 80.2 \\
MLVCNN~\cite{yu2018multi} & 92.2 \\
MVTN~\cite{hamdi2021mvtn} & 92.9 \\
Latformer~\cite{he2024latformer} & 97.4\\
\midrule
\taco~ (ours) & \textbf{99.33}\\
\bottomrule
\end{tabular}
\caption{Shape retrieval (mAP) results on ModelNet40.}
\label{tab:modelnet40_retrieval}
\end{table}

\clearpage

\end{document}